\documentclass[10pt,twocolumn,twoside]{IEEEtran}
\makeatletter
\let\NAT@parse\undefined
\makeatother

\usepackage{hyperref}

\usepackage{cite}
\usepackage{amsmath,amssymb,amsfonts}
\usepackage{mathtools}
\usepackage{stmaryrd} 

\usepackage{xcolor}
\usepackage[ruled, vlined,linesnumbered]{algorithm2e}
\usepackage{paralist}
\usepackage{booktabs}
\usepackage{array}
\usepackage[utf8]{inputenc}
\usepackage[capitalize]{cleveref}
\usepackage{graphicx}
\usepackage{textcomp}
\usepackage{tensor} 
\definecolor{fimdpcolor}{rgb}{0.00000,0.44700,0.74100}
\definecolor{stormcolor}{rgb}{0.74100,0.00000,0.44700}
\definecolor{darkgreen}{rgb}{0.0078,0.5450, 0.0666}
\definecolor{darkblue}{rgb}{0.1529,0.5019,0.8705}
\definecolor{darkgray}{rgb}{0.45,0.45,0.45}

\usepackage{xifthen}

\usepackage{tikz}
\usetikzlibrary{myautomata}
\usetikzlibrary{cmdps}
\usetikzlibrary{gridworld}
\usetikzlibrary{decorations.pathreplacing}

\usepackage{pgfplots}
\pgfplotsset{compat=1.16}

\usepackage{amsthm}
\newtheorem{theorem}{Theorem}
\newtheorem{definition}{Definition}
\newtheorem{lemma}{Lemma}
\newtheorem{example}{Example}

\newlength{\spc} 

\newcommand{\punct}[1]{%
  \settowidth{\spc}{#1}
  \addtolength{\spc}{-1.8\spc}
  #1
  \hspace*{\spc}
}

\newcommand{\optmdp}[2][]{
  \ifthenelse{\isempty{#1}}%
    {#2}     
    {#2_{#1}}
}
\newcommand{\loaded}[2]{\tensor[^{#2}]{#1}{}}

\newcommand{\Nset}{\mathbb{N}}
\newcommand{\extNset}{\overline{\Nset}}

\renewcommand{\vec}[1]{\mathbf{#1}}
\newcommand{\veccomp}[2]{#1(#2)}
\newcommand{\infvec}{\boldsymbol{\infty}} 

\newcommand{\loadedprobm}[4][]{\tensor*[^{#4}_{#3}]{\mathbb{P}}{_{#1}^{#2}}}
\newcommand{\calO}{\mathcal{O}}

\newcommand{\mdp}{\mathcal{C}}
\newcommand{\states}{S}
\newcommand{\actions}{A}
\newcommand{\trans}{\Delta}
\newcommand{\cons}{\gamma}
\newcommand{\reloads}{\mathit{R}}
\newcommand{\Ca}{\mathit{cap}}
\newcommand{\altmdp}[3][]{\ifthenelse{\isempty{#1}}{\mathcal{B}(#2,#3)}{\mathcal{B}(#2,#3,#1)}}
\newcommand{\sink}{\mathit{sink}}

\newcommand{\altarget}{X}

\newcommand{\acta}{\mathsf{a}}
\newcommand{\actb}{\mathsf{b}}
\newcommand{\hacta}{\textcolor{magenta!80!black}{\acta}}
\newcommand{\hactb}{\textcolor{cyan!80!black}{\actb}}

\newcommand{\Succ}{\mathit{Succ}}
\newcommand{\target}{T}

\newcommand{\trunc}[2][]{\optmdp[#1]{\llbracket \, #2 \, \rrbracket}}
\newcommand{\strunc}[2][]{\optmdp[#1]{\llfloor \,#2 \,\rrfloor}}

\newcommand{\apath}{\alpha}
\newcommand{\run}{\mathit{\varrho}}
\newcommand{\hist}{\alpha}
\newcommand{\histpr}{\beta}
\newcommand{\histconc}{\odot}
\newcommand{\len}[1]{len(#1)}
\newcommand{\pref}[1]{_{..#1}}
\newcommand{\suff}[1]{_{#1..}}
\newcommand{\infix}[2]{_{#1..#2}}
\newcommand{\rstate}[2][\run]{#1_{#2}}
\newcommand{\last}[1]{\mathit{last(#1)}}
\newcommand{\loadedpath}[2]{\loaded{#1}{#2}}
\newcommand{\reslevs}[3][]{\optmdp[#1]{\mathit{RL}}(\loadedpath{#3}{#2})}
\newcommand{\reslevsscr}[4][]{\optmdp[#1]{\mathit{RL}}(\tensor*[^{#2}]{#3}{#4})}
\newcommand{\lreslev}[3][]{\mathit{last}\reslevs[#1]{#2}{#3}}

\newcommand{\lcompatible}[4][]{\optmdp[#1]{\mathsf{Comp}}(#2,#3,#4)}
\newcommand{\selector}{\Phi}
\newcommand{\selrule}{\varphi}
\newcommand{\dom}{\mathit{dom}}
\newcommand{\srules}[1][]{\optmdp[#1]{\mathit{Rules}}}
\newcommand{\selection}[2]{\mathit{select}(#1,#2)}

\newcommand{\obj}{\mathsf{O}}
\newcommand{\safety}{\mathsf{S}}
\newcommand{\reachability}[1][]{\mathsf{R}_{#1}}
\newcommand{\buchi}[1][]{\mathsf{B}_{#1}}
\newcommand{\nonreloading}[1][]{\mathsf{N}_{#1}}

\newcommand{\minreach}[2][]{\mathsf{F}_{\!#1}^{#2}}
\newcommand{\sdsat}[6][]{#2 \!\tensor*[_{#3}^{#4}]{\models}{_{#1}^{\!#6}}\!#5}
\newcommand{\vecsat}[5][]{#2 \!\tensor*[^{#3}]{\models}{_{#1}^{\!#5}}\!#4}
\newcommand{\pos}{>0}
\newcommand{\as}{=1}
\newcommand{\ml}[3][]{\vec{ml}[#2]_{#1}^{#3}}

\newcommand{\AVname}{\mathit{AV}}
\newcommand{\AV}[3]{\AVname(#1,#2,#3)}

\newcommand{\shopeName}{\mathit{HV}}
\newcommand{\shope}[4]{\shopeName(#1, #2, #3, #4)}
\newcommand{\shopeSinks}[5]{\shopeName[#5](#1, #2, #3, #4)}
\newcommand{\SVname}{\mathit{SV}}
\newcommand{\SV}[3]{\SVname(#1, #2, #3)}
\newcommand{\SVSinks}[4]{\SVname[#4](#1, #2, #3)}
\newcommand{\SVThres}[4]{\SVname_{\!#4}(#1, #2, #3)}

\newcommand{\reachindex}{reachability index}

\newcommand{\sinkstates}[1]{F(#1)}

\newcommand{\ERT}[5][]{\mathit{ERT}_{#1}(#2, #3, #4, #5)}

\newcommand{\old}{\mathit{old}}
\newcommand{\varRel}{\reloads'}
\newcommand{\varToRemove}{\mathit{Unusable}}

\newcommand{\minsafe}{min-safe}

\newcommand{\FiMDP}{\textsc{FiMDP}}
\newcommand{\FiMDPEnv}{\textsc{FiMDPEnv}}

\newcommand{\Storm}{\textsc{Storm}}
\newcommand{\Stormpy}{\textsc{Stormpy}}
\newcommand{\Jani}{JANI}
\newcommand{\Prism}{PRISM}

\title{Efficient Strategy Synthesis for MDPs\\ with Resource Constraints}

\author{František Blahoudek,
Petr Novotný,
Melkior Ornik,
Pranay Thangeda, and
Ufuk Topcu
\thanks{%
  Submitted for review on 26 January, 2021. This work was partially supported by NASA’s Space Technology Research Grants program for Early Stage Innovations under the grant ``Safety-Constrained and Efficient Learning for Resilient Autonomous Space Systems'', by DARPA's grant HR001120C0065, and by the Czech Ministry of Education by ``Youth and Sports'' ERC.CZ project LL1908. Petr Novotný is supported by the Czech Science Foundation Junior grant no. GJ19-15134Y.}%
  \thanks{František Blahoudek was with the Oden Institute, The University of Texas at Austin, Austin, USA. He is now with the Faculty of Information Technology, Brno University of Technology, Brno, Czech Republic (e-mail: frantisek.blahoudek@gmail.com).}
  \thanks{Petr Novotný is with the Faculty of Informatics, Masaryk University, Brno, Czech Republic (e-mail: petr.novotny@fi.muni.cz).}
  \thanks{Melkior Ornik and Pranay Thangeda are with the Department of Aerospace Engineering, University of Illinois at Urbana-Champaign, Urbana, USA (e-mail: mornik@illinois.edu, pranayt2@illinois.edu).}
  \thanks{Ufuk Topcu is with the Department of Aerospace Engineering and Engineering Mechanics, The University of Texas at Austin, Austin, USA (e-mail: utopcu@utexas.edu).}%
}

\begin{document}

\maketitle

\begin{abstract}
We consider qualitative strategy synthesis for the formalism called
consumption Markov decision processes. This formalism can model dynamics of an
agents that operates under resource constraints in a stochastic environment.
The presented algorithms work in time polynomial with respect to the
representation of the model and they synthesize strategies ensuring that a
given set of goal states will be reached (once or infinitely many times) with
probability 1 without resource exhaustion. In particular, when the amount of
resource becomes too low to safely continue in the mission, the strategy
changes course of the agent towards one of a designated set of reload states
where the agent replenishes the resource to full capacity; with sufficient
amount of resource, the agent attempts to fulfill the mission again.

We also present two heuristics that attempt to reduce expected time that the
agent needs to fulfill the given mission, a parameter important in practical
planning. The presented algorithms were implemented and numerical examples
demonstrate (i) the effectiveness (in terms of computation time) of the
planning approach based on consumption Markov decision processes and (ii) the
positive impact of the two heuristics on planning in a realistic example.
\end{abstract}

\begin{IEEEkeywords}
consumption Markov decision process,
planning,
resource constraints,
strategy
synthesis
\end{IEEEkeywords}

\section{Introduction}

Autonomous agents like driverless cars, drones, or planetary rovers typically
operate under resource constraints. A lack of the critical resource usually
leads to a mission failure or even to a crash.

Autonomous agents are often deployed in stochastic environments which exhibit
uncertain outcomes of the agents' actions. \emph{Markov decision processes
(MDPs)} are commonly used to model such environments for planning purposes.
Intuitively, an MDP is described by a set of \emph{states} and
\emph{transitions} between these states. In a discrete-time MDP, the evolution
happens in discrete steps and a transition has two phases: first, the agent
chooses some \emph{action} to play, and the resulting state is chosen randomly
based on a probability distribution defined by the action and the agent's
state.

The interaction of an agent with an MDP is formalized using strategies. A
strategy is simply a recipe that tells the agent, in every moment, what action
to play next. The problem of finding strategies suitable for given objectives
is called \emph{strategy synthesis for MDPs}.

As the main results of this paper, we solve strategy synthesis for two kinds
of objectives in resource-constrained MDPs. These two objectives are (i)
\emph{almost-sure reachability} of a given set of states $\target$, and (ii)
\emph{almost-sure Büchi objective} for $\target$. That is, the synthesized
strategies ensure that, with probability 1 and without resource exhaustion,
some target from $\target$ will be reached at least once or $\target$ will be
visited infinitely often.

We also present two heuristics that improve the practical utility of the
presented algorithms for planning in resource-constrained systems. In
particular, the \emph{goal-leaning} and \emph{threshold} heuristics attempt,
as a secondary objective, to reach $\target$ in a short time. Further, we
briefly describe our tool implementing these algorithms and we demonstrate
that our approach specialized to qualitative analysis of resource-constrained
systems can solve this task faster then the state-of-the-art general-purpose
probabilistic model checker \Storm~\cite{storm}.

\subsection{Current approaches to resource-constraints.}

There is a substantial body of work in the area of verification of
resource-constrained systems~\cite{CdAHS:resource-interfaces,
BFLMS:weak-upper-bound,BHR:battery, BBFLMR:energy-controller-synthesis,
WHLK:energy-aware-scheduling, SSDNLB:energy-validation,
FL:featured-weighted-automata, FZ:cost-parity-games,
BDKL:energy-utility-probabilistic-mc, BDDKK:energy-utility-quantiles}. %
A naive approach is to model such systems as finite-state systems with states
augmented by an integer variable representing the current \emph{resource
level}. The resource-constraint requires that the resource level never drops
below zero.

The well-known \emph{energy}
model~\cite{CdAHS:resource-interfaces,BFLMS:weak-upper-bound} avoids the
encoding of the resource level into state space: instead, the model uses a
integer counter, transitions are labeled by integers, and taking an
$\ell$-labelled transition results in $\ell$ being added to the counter. Thus,
negative numbers stand for resource consumption while positive ones represent
charging. Many variants of both MDP and game-based energy models have been
studied. In particular,~\cite{CHD:energy-MDPs} considers strategy synthesis
for energy MDPs with qualitative Büchi and parity objectives. The main
limitation of the energy models is that in general, they are not known to
admit strategy synthesis algorithms that work in time polynomial with respect
to the representation of the model. Indeed, already the simplest problem,
deciding whether a non-negative energy can be maintained in a two-player
energy game, is at least as hard as solving mean-payoff graph
games~\cite{BFLMS:weak-upper-bound}; the question whether the latter belongs
to $\mathsf{P}$ is a well-known open
problem~\cite{Jurdzinski:parity-to-mp}. This hardness translates also to
MDPs~\cite{CHD:energy-MDPs}, making polynomial-time strategy synthesis for
energy MDPs impossible without a theoretical breakthrough.

\subsection{Consumption MDPs}

Our work is centered around \emph{Consumption MDPs (CMDPs)} which is a model
motivated by a real-world vehicle energy consumption and inspired by
consumption games~\cite{BCKN:consumption-games}. In a CMDP, the agent has a
finite storage \emph{capacity}, each action consumes a non-negative amount of
resource, and replenishing of the resource happens only in a designated
states, called \emph{reload states}, as an atomic (instant) event. In
particular, the resource levels are kept away from the states.

Reloading as atomic events and bounded capacity are the key ingredients for
efficient analysis of CMDPs. Our qualitative strategy synthesis algorithms
work provably in time that is polynomial with respect to the representation of
the model. Moreover, they synthesize strategies with a simple structure and an
efficient representation via binary counters.

We first introduced CMDPs and presented the algorithm for the Büchi objective
in~\cite{Blaetal20}. In contrast to~\cite{Blaetal20}, this paper contains the
omitted proofs, it extends the algorithmic core with the reachability
objective and it introduces \emph{goal-leaning} and \emph{threshold}
heuristics that attempt to improve expected reachability time of targets.
Moreover, the presentation in this manuscript is based on new notation that
simplifies understanding of the merits and proofs, it uses more pictorial
examples, and finally, we provide a numerical example that uses \Storm{} as a
baseline for comparison.

\subsection{Outline}%
\Cref{sec:prelims} introduces CMDPs with the necessary notation and it is
followed by \cref{sec:counter} which discusses strategies with binary counters.
\Cref{sec:safety,sec:posreach} solve two intermediate objectives for CMDPs,
namely \emph{safety} and \emph{positive reachability}, that serve as stepping
stones for the main results. The solution for the Büchi objective is
conceptually simpler than the one for almost-sure reachability and thus is
presented first in \cref{sec:buchi}, followed by \cref{sec:as-reach} for the
latter. \Cref{sec:heuristics} defines expected reachability time and proposes
the two heuristics for its reduction. Finally, \cref{sec:experiments}
describes briefly our implementation and two numerical examples: one showing
the effectiveness of CMDPs for analysis of resource-constrained systems and
one showing the impact of the proposed heuristics on expected reachability
time. For better readability, two rather technical proofs were moved from
\cref{sec:posreach} to Appendix.
\section{Preliminaries}
\label{sec:prelims}

We denote by $\Nset$ the set of all non-negative integers and by $\extNset$
the set $\Nset \cup \{\infty\}$. For a set $I$ and a vector
$\vec{v}\in\extNset^{I}$ indexed by $I$ we use $\veccomp{\vec{v}}{i}$ for the
$i$-component of $\vec{v}$. We assume familiarity with basic notions of
probability theory.

\subsection{Consumption Markov decision processes (CMDPs)}

\begin{definition}[CMDP]
A \emph{consumption Markov decision process} (CMDP) is a tuple %
$\mdp = (\states, \actions, \trans, \cons, \reloads, \Ca)$ where $\states$ is
a finite set of  \emph{states}, %
$\actions$ is a finite set of \emph{actions}, %
$\trans\colon \states\times \actions \times \states \rightarrow [0,1]$ is a
\emph{transition function} such that for all $s\in\states$ and $a\in\actions$
we have that $\sum_{t\in\states}\trans(s,a,t) = 1$, %
$\cons\colon \states \times \actions \rightarrow \Nset$ is a \emph{consumption
function}, %
$\reloads\subseteq \states$ is a set of \emph{reload states} where the
resource can be reloaded, and %
$ \Ca $  is a \emph{resource capacity}.
\end{definition}

\textbf{Visual representation.} CMDPs are visualized as shown in
\cref{fig:example} for a CMDP $(\{r, s, t, u, v\}, \{\acta, \actb\}, \trans, \cons,
\{r, t\}, 20)$. States are circles, reload states are double circled, and
target states (used later for reachability and Büchi objectives) are
highlighted with a green background. Capacity is given in the yellow box. The
functions $\trans$ and $\cons$ are given by (possibly branching) edges in the
graph. Each edge is labeled by the name of the action and by its consumption
enclosed in brackets. Probabilities of outcomes are given by gray labels in
proximity of the respective successors. For example, the cyan branching edge
stands for $\trans(s,\actb,u)=\trans(s,\actb,v)=\frac{1}{2}$ and $\cons(s,\actb)=5$. %
To avoid clutter, we omit $\textcolor{gray}{1}$ for non-branching edges and we
merge edges that differ only in action names and otherwise are identical. As
an example, the edge from $r$ to $s$ means that $\trans(r,\acta,s) = \trans(r,\actb,s)
= 1$ and $\cons(r,\acta) = \cons(r,\actb) = 1$. The colors of edges do not carry any
special meaning are used later in text for easy identification of the
particular actions.

\begin{figure}[h]
\centering
\begin{tikzpicture}[automaton, xscale=1.5]
\tikzstyle{a} = [magenta!80!black]
\tikzstyle{b} = [cyan!80!black]
\tikzstyle{ab} = [black, thick]
\node[state] (start) at (0,0)  {$s$};
\node[state,reload,target] (t)   at (2.5,-1) {$t$};
\node[state,reload] (rel)   at (-1.5,0) {$r$};
\node[state] (mis)      at (2.5,1)  {$u$};
\node[state] (pathhome) at (1,1)  {$v$};

\path[->,auto,swap]
(start) edge[bend right=80, a]
  node[] {$\hacta$}
  node[cons, swap] {2}
(rel)
(rel) edge[bend right=80, ab]
  node[swap] {$\acta,\actb$}
  node[cons] {1}
(start)
(start)
  edge[out=0, in=120,looseness=1.1, b]
    node[cons] {5}
    node[prob,below left,pos=.8] {$\frac{1}{2}$}
  (t)
  edge[out=0, in=240,looseness=1.1, b]
    node[swap, pos=.46] {$\hactb$}
    node[prob,above left,pos=.8] {$\frac{1}{2}$}
  (mis)
(mis) edge[ab]
  node {$\acta,\actb$}
  node[cons,swap, pos=.6] {1}
(pathhome)
(pathhome) edge[ab, bend right]
  node {$\acta,\actb$}
  node[cons, swap, pos=.5] {2}
(start)
(t) edge[loop above, ab, looseness=16]
  node[] {$\acta,\actb$}
  node[cons, right, outer sep=5pt] {1}
(t)
;

\node[capacity] at (1, -1) {20};
\end{tikzpicture}
\caption{A CMDP with a target set $\{t\}$.}%
\label{fig:example}
\end{figure}
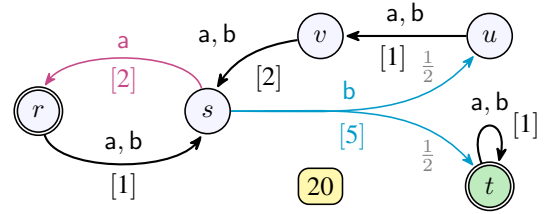

For $s\in \states$ and $a\in \actions$, we denote by $\Succ(s,a)$ the set
$\{t\mid \trans(s,a,t)>0\}$. A \emph{path} is a (finite or infinite)
state-action sequence $\hist=s_1a_1s_2a_2s_3\dots\in (\states\cdot
\actions)^\omega \cup (\states\cdot \actions)^*\cdot\states$ such that
$s_{i+1}\in \Succ(s_i,a_i)$ for all $i$. We define $\rstate[\hist]{i}=s_i$ and we say that $\apath$ is \emph{$s_1$-initiated}. We use
$\hist\pref{i}$ for the finite prefix $s_1 a_1 \dots s_i$ of $\hist$, $
\hist\suff{i} $ for the suffix $ s_i a_i\dots $, and $\hist\infix{i}{j}$ for
the infix $s_ia_i\ldots s_j$. A finite path is a \emph{cycle} if it starts and
ends in the same state and is \emph{simple} if none of its proper infixes
forms a cycle. The \emph{length} of a path $\hist$ is the number $\len{\hist}$
of actions on $\hist$, with $\len{\hist}=\infty$ if $\hist$ is infinite.

An infinite path is called a \emph{run}. We typically name runs by variants of
the symbol $\run$. A finite path is called \emph{history}. We use
$\last{\hist}$ for the last state of a history $\hist$. For a history $\hist$
with $\last{\hist}=s_1$ and for $\histpr=s_1a_1s_2a_2\ldots$ we define a
\emph{joint path} as $\hist\histconc\histpr = \hist a_1s_2a_2\ldots$.

A CMDP is \emph{decreasing} if for every cycle $ s_1 a_1s_2 \ldots
a_{k-1}s_k $ there exists $ 1 \leq i < k $ such that  $ \cons(s_i,a_i)
>0 $. Throughout this paper we consider only decreasing CMDPs. The
only place where this assumption is used are the proofs
of~\Cref{thm:safety-main} and \Cref{thm:buchi}.

\subsection{Resource: consumption and levels}%

The semantics of the consumption $\cons$, reload states $R$ and capacity $\Ca$
naturally capture evolution of levels of the resource along paths in $\mdp$.
Intuitively, each computation of $\mdp$ must start with some \emph{initial
load} of the resource, actions consume the resource, and reload states
replenish the resource level to $\Ca$. The resource is depleted if its level
drops below $0$, which we indicate by the symbol $\bot$ in the following.

Formally, let $\apath = s_1a_1s_2\ldots s_n$ (where $n$ might be $\infty$) be a
path in $\mdp$ and let $0\leq d \leq \Ca$ be an \emph{initial load}. We write
$\loadedpath\apath{d}$ to denote the fact that $\apath$ started with $d$ units
of the resource. We say that $\apath$ is \emph{loaded} with $d$ and that
$\loadedpath{\apath}{d}$ is a \emph{loaded path}. The \emph{resource levels of
$\loadedpath{\apath}{d}$} is the sequence $\reslevs[\mdp]{d}{\apath} =
r_1,r_2,\ldots,r_n$ where $r_1 = d$ and for $1 \leq i < n$ the next resource
level $r_{i+1}$ is defined inductively, using $c_i=\cons(s_i, a_i)$ for the
consumption of $a_i$, as
\[
r_{i+1} =
\begin{cases}
  r_i - c_i &
    \text{if } s_i \not\in\reloads \text{ and }%
    c_i \leq r_i \neq \bot\text{,}\\
  \Ca - c_i &
    \text{if } s_i \in \reloads \text{ and }%
    c_i\leq \Ca \text{ and }%
    r_i \neq \bot\text{,}\\
  \bot & \text{otherwise}.
\end{cases}
\]
If $\apath$ (and thus $n$) is finite, we use $\lreslev[\mdp]{d}{\apath}$ to
reference the last resource level $r_n$ of $\loadedpath{\apath}{d}$.

A loaded path $\loadedpath{\apath}{d}$ is \emph{safe} if $\bot$ is not present
in $\reslevs[\mdp]{d}{\apath}$, which we write as
$\bot\notin\reslevs[\mdp]{d}{\apath}$. Naturally, if $\loadedpath{\apath}{d}$ is
safe then $\loadedpath{\apath}{h}$ is safe for all $h \geq d$.

\begin{example}\label{ex:reslev}
Consider the CMDP in \cref{fig:example} with capacity $20$ and the run
$\run=(s\hacta r \acta)^\omega$ with the initial load $2$. We have that
$\reslevs[\mdp]{2}{\run}=2,0,19,17,19,17\ldots$ and thus
$\loadedpath{\run}{2}$ is safe. On the other hand, for the run
$\run'=(s\hactb u \acta v \acta)^\omega$ we have
$\reslevsscr[\mdp]{20}{\run}{^\prime} = 20,15,14,12,7,6,4,\bot,\bot,\ldots$
and, in fact, no initial load can make $\run'$ safe.
\end{example}

\subsection{Strategies}%
A \emph{strategy} $\sigma$ for $\mdp$ is a function assigning an action to
each loaded history. An evolution of $\mdp$ under the control of $\sigma$
starting in some initial state $s\in\states$ with an initial load $d \leq \Ca$
creates a loaded path $\loadedpath{\apath}{d}={}^ds_1a_1s_2\ldots$ as follows.
The path starts with $s_1=s$ and for $i \geq 1$ the action $a_i$ is selected
by the strategy as $a_{i}=\sigma({}^ds_1a_1s_2\ldots s_i)$, and the next state
$s_{i+1}$ is chosen randomly according to the values of $\trans(s_i, a_{i},
\cdot)$. Repeating this process \emph{ad infinitum} yields an infinite sample
run (loaded by $d$). Loaded runs created by this process are
\emph{$\sigma$-compatible}. We denote the set of all $\sigma$-compatible
$s$-initiated runs loaded by $d$ by $\lcompatible[\mdp]{\sigma}{s}{d}$.


We denote by $\loadedprobm[\mdp]{\sigma}{s}{d}(\mathsf{A})$ the probability that
a sample run from $\lcompatible[\mdp]{\sigma}{s}{d}$ belongs to a given
measurable set of loaded runs $\mathsf{A}$. For details on the formal
construction of measurable sets of runs see~\cite{Ash:book}.

\subsection{Objectives and problems}

A \emph{resource-aware objective} (or simply an \emph{objective}) is a set of
loaded runs. The objective $\safety$ (\emph{safety}) contains exactly all
loaded runs that are safe. Given a \emph{target set} $\target\subseteq
\states$ and $i \in \Nset$, the objective $\reachability[\target]^i$ (\emph{bounded
reachability}) is the set of all safe loaded runs that reach some state from
$\target$ within the first $i$ steps, which is $\reachability[\target]^i = \{
\loadedpath{\run}{d} \in \safety \mid \run_j \in \target \text{ for some } 1
\leq j \leq i + 1 \}$. The union $\reachability[\target] = \bigcup_{i \in
\Nset} \reachability[T]^i$ forms the \emph{reachability} objective. Finally,
the objective $\buchi[\target]$ (\emph{Büchi}) contains all safe loaded runs
that visit $\target$ infinitely often.%

The safety objective --- never depleting the critical resource --- is
of primary concern for agents in CMDPs. We reflect this fact in the following
definitions. Let us now fix a target set $\target\subseteq \states$, a state
$s\in \states$, an initial load $d$, a strategy $\sigma$, and an objective
$\obj$. We say that $\sigma$ loaded with $d$ in $s$
\begin{itemize}%
\item \emph{satisfies $\obj$ surely}, written as
$\sdsat[\mdp]{\sigma}{s}{d}{\obj}{}$, if and only if $\loadedpath{\run}{d}\in\obj$ holds for every $\loadedpath{\run}{d}\in\lcompatible[\mdp]{\sigma}{s}{d}$;%
\item \emph{safely satisfies $\obj$ with positive probability}, written as
$\sdsat[\mdp]{\sigma}{s}{d}{\obj}{\pos}$, if and only if
$\sdsat[\mdp]{\sigma}{s}{d}{\safety}{}$ and
$\loadedprobm[\mdp]{\sigma}{s}{d}(\obj) > 0$;
\item \emph{safely satisfies $\obj$ almost surely}, written as
$\sdsat[\mdp]{\sigma}{s}{d}{\obj}{=1}$, if and only if
$\sdsat[\mdp]{\sigma}{s}{d}{\safety}{}$ and
$\loadedprobm[\mdp]{\sigma}{s}{d}(\obj) = 1$.%
\end{itemize}

We naturally extend the satisfaction relations to strategies loaded by vectors.
Let $\vec{x}\in \extNset^\states$ be a vector of initial loads. The strategy
$\sigma$ loaded by $\vec{x}$ satisfies $\obj$, written as
$\vecsat[\mdp]{\sigma}{\vec{x}}{\obj}{}$, if and only if
$\sdsat[\mdp]{\sigma}{s}{\veccomp{\vec{x}}{s}}{\obj}{}$ holds for all
$s\in\states$ with $\veccomp{\vec{x}}{s} \neq \infty$. We extend the other two
relations analogously to $\vecsat[\mdp]{\sigma}{\vec{x}}{\obj}{\pos}$ and
$\vecsat[\mdp]{\sigma}{\vec{x}}{\obj}{=1}$.

The vector $\ml[\mdp]{\obj}{}$ is the component-wise minimal vector for which
there exists a strategy $\pi$ such that $\vecsat[\mdp]{\pi}{\ml{\obj}{}}{\obj}{}$.
We call $\pi$ the \emph{witness strategy} for $\ml[\mdp]{\obj}{}$. If
$\veccomp{\ml[\mdp]{\obj}{}}{s}=\infty$, no strategy satisfies $\obj$ from $s$
even when loaded with $\Ca$. Vectors $\ml[\mdp]{\obj}{\pos}$ and
$\ml[\mdp]{\obj}{=1}$ are defined analogously using $\vecsat{}{}{}{\pos}$ and
$\vecsat{}{}{}{=1}$, respectively.

We consider the following qualitative problems for CMDPs: \emph{Safety},
\emph{positive reachability} \emph{almost-sure Büchi}, and \emph{almost-sure
reachability} which equal to computing $\ml[\mdp]{\safety}{}$,
$\ml[\mdp]{\reachability}{\pos}$, $\ml{\buchi}{\as}$, and
$\ml[\mdp]{\reachability}{\as}$, respectively, and the corresponding witness
strategies. The solutions of the latter two problems build on top of the
first two.

\subsection{Additional notation and conventions}

For given $\reloads'\subseteq \states$, we denote by $\mdp(R')$ the CMDP that
uses $R'$ as the set of reloads and otherwise is defined as $\mdp$. Throughout
the paper, we drop the subscripts $\mdp$ and $\target$ in symbols whenever
$\mdp$ or $\target$ is known from the context.

Calligraphic font (e.g. $\mdp$) is used for names of CMDPs, sans serifs (e.g. $\safety$) is used for
objectives (set of loaded runs), and vectors are written in bold. Action names
are letters from the start of alphabet, while states of CMDPs are usually
taken from the later parts of alphabet (starting with $r$). The symbol
$\apath$ is used for both finite and infinite paths, and $\run$ is only used
for infinite paths (runs). Finally, strategies are always variants of
$\sigma$ or $\pi$.

\subsection{Strategies revisited}

A strategy $\sigma$ is \emph{memoryless} if $ \sigma(\loadedpath{\hist}{d})
= \sigma(\loadedpath{\histpr}{h}) $ whenever $ \last{\hist}=\last{\histpr} $.

\begin{example}\label{ex:resource_strategy}
The runs $\run$ and $\run'$ from \cref{ex:reslev} are sample runs created by
two different memoryless strategies: $\sigma_{\hacta}$ that always picks
$\hacta$ in $s$, and $\sigma_{\hactb}$ that always picks $\hactb$ in $s$,
respectively. As $\run$ is the only $s$-initiated run of $\sigma_{\hacta}$, we
have that $\sdsat{\sigma_{\hacta}}{s}{2}{\safety}{}$. However,
$\sigma_{\hacta}$ is not useful if we attempt to eventually reach $t$ and we
clearly have $\loadedprobm{\sigma_{\hacta}}{s}{2}(\reachability[\{t\}]) = 0$.
On the other hand, $\run'$ is the witness for the fact that $\sigma_{\hactb}$
does not even satisfy the safety objective for any initial load.
As we have no other choice in $s$, we can conclude that memoryless strategies
are not sufficient in our setting. Consider instead a strategy $\pi$ that
picks $\hactb$ in $s$ whenever the current resource level is at least $10$ and
picks $\hacta$ (and reloads in $r$) otherwise. Loaded with $2$ in $s$, $\pi$
satisfies safety and it guarantees reaching $t$ with a positive probability:
in $s$, we need at least 10 units of resource to return to $r$ in the case we
are unlucky and $\hactb$ leads us to $u$; if we are lucky, $\hactb$ leads us
directly to $t$, witnessing that
$\loadedprobm{\pi}{s}{2}(\reachability[\{t\}]) > 0$. Moreover, at
\emph{every} revisit of $r$ there is a $\frac{1}{2}$ chance of hitting $t$
during the next attempt, which shows that
$\sdsat{\pi}{s}{2}{\reachability[\{t\}]}{=1}$.
\end{example}

\textbf{Remark.} While computing the sure satisfaction relation $\models$ on a
CMDP follows similar approaches as used for solving a consumption 2-player
game~\cite{BCKN:consumption-games}, the solutions for $\sdsat{}{}{}{}{\pos}$
and $\sdsat{}{}{}{}{=1}$ differ substantially. Indeed, imagine that, in the
CMDP from \cref{fig:example}, the outcome of the action $\hactb$ from state $s$
is resolved by an adversarial player (who replaces the random resolution). The
player can always pick $u$ as the next state and then the strategy $\pi$ does
not produce any run that reaches $t$. In fact, there would be no strategy that
guarantees reaching $t$ against such a player at all.

\medskip

The strategy $\pi$ from \cref{ex:resource_strategy} uses finite memory to
track the resource level exactly. Under the standard definition, a strategy is
a \emph{finite memory} strategy, if it can be encoded by a \emph{memory
structure}, a type of finite transducer (a finite state machine with outputs).
Tracking resource levels using states in transducers is memory-inefficient.
Instead, the next section introduces resource-aware strategies that rely on
binary counters to track resource levels.

\section{Strategies with binary counters}%
\label{sec:counter}

In this section, we define a succinct representation of finite-memory
strategies suitable for CMDPs. Let us fix a CMDP $\mdp = (\states, \actions,
\trans, \cons, \reloads, \Ca)$ for the rest of this section. In our setting,
strategies need to track resource levels of histories. A non-exhausted
resource level is always a number between $0$ and $\Ca$, which can be
represented with a binary-encoded bounded counter. A binary-encoded counter
needs $\log_2 \Ca$ bits of memory to represent numbers between $0$ and $\Ca$
(the same as integer variables in computers). Representation of resource
levels using states in transducers would require $\Ca$ states.

We call strategies with such binary-encoded counters \emph{finite counter
strategies}. In addition to the counter, a finite counter strategy needs
\emph{rules} that select actions based on the current resource level, and a
\emph{rule selector} that pick the right rule for each state.

\begin{definition}[Rule]
A \emph{rule} $\selrule$ for $\mdp$ is a partial function
from the set $\{0,\ldots,\Ca\}$ to $A$. An undefined value for some
$n$ is indicated by $\selrule(n)=\bot$.
\end{definition}

We use $\dom(\selrule) = \{n \in \{0 ,\ldots,\Ca\} \mid \selrule(n)\neq \bot\}$
to denote the domain of $\selrule$ and we call the elements of $\dom(\selrule)$
\emph{border levels}. We use $\srules[\mdp]$ for the set of all rules
for $\mdp$.

A rule compactly represents a total function using intervals.
Intuitively, the selected action is the same for all values of the resource
level in the interval between two border levels. Formally, let $l$ be the
current resource level and let $n_1 < n_2 < \cdots < n_k$ be the border levels
of $\varphi$ sorted in the ascending order. Then the selection according to
rule $\selrule$ for $l$, written as $\selection{\selrule}{l}$, picks the
action $\selrule(n_i)$, where $n_i$ is the largest border level such that $n_i
\leq l$. In other words, $\selection{\selrule}{l} = \selrule(n_i)$ if the
current resource level $l$ is in $[n_i,n_{i+1})$ (putting $n_{k+1} = \Ca +
1$). We set $\selection{\selrule}{l} = a$ for some globally fixed action
$a\in\actions$ (for completeness) if $l < n_1$. In particular,
$\selection{\selrule}{\bot} = a$.

\begin{definition}[Rule selector]
A \emph{rule selector} for $\mdp$ is a function $\selector\colon
\states \rightarrow \srules$.
\end{definition}

A binary-encoded counter that tracks the resource levels of paths together
with a rule selector $\selector$ encode a strategy $\sigma_{\selector}$. Let
$\loadedpath{\hist}{d}=\loadedpath{s}{d}_1a_1s_2\ldots s_n$ be a loaded
history. We assume that we can access the value of $\lreslev{d}{\hist}$ from
the counter and we set
\[
\sigma_\selector(\loadedpath{\hist}{d}) = \selection{\selector(s_n)}{\lreslev{d}{\hist}}.
\]

A strategy $\sigma$ is a \emph{finite counter strategy} if there is a rule
selector $\selector$ such that $\sigma = \sigma_\selector$. The rule selector
can be imagined as a device that implements $\sigma$ using a table of size $ \calO(|\states|) $, where the size of each
cell
corresponds to the number of border levels times $ \calO(\log \Ca) $ (the latter representing the number of bits
required to encode a level). In particular, if the total number of border levels $ \selector $ is polynomial in the
size
of the MDP, so is the number of bits required to represent $ \selector $ (and thus, $ \sigma_\selector $). This
contrasts with the traditional representation of finite-memory strategies via
\emph{transducers}~\cite{Apt:2011:LecturesGameTheory}, since transducers would require at least $ \Theta(\Ca)$ states
to keep track of the current resource level.

\begin{example}\label{ex:selector}
Consider again the CMDP from \cref{fig:example}. Let $\selrule$ be a
rule with $\dom(\selrule) = \{0,10\}$ such that $\selrule(0) =
\hacta$ and $\selrule(10) = \hactb$, and let $\selrule'$ be a rule
with $\dom(\selrule')=\{0\}$ such that $\selrule(0)=\acta$. Finally, let
$\selector$ be a rule selector such that $\selector(s) = \selrule$ and
$\selector(s') = \selrule'$ for all $s \neq s' \in \states$. Then, the
strategy $\pi$ informally described in \cref{ex:resource_strategy} can
be formally represented by putting $\pi = \sigma_\selector$. Note that for any
$\vec{i}$ with $\vec{i}(s)\geq 2$, $\vec{i}(t)\geq 0$, $\vec{i}(u)\geq 5$,
$\vec{i}(v)\geq 4$, and $\vec{i}(r)\geq 0$ we have that
$\vecsat{\pi}{\vec{i}}{\reachability[\{t\}]}{=1}$.
\end{example}
\section{Safety}\label{sec:safety}%

In this section, we present \cref{algo:safety} that computes $\ml{\safety}{}$
and the corresponding witness strategy. Such a strategy guarantees that, given
a sufficient initial load, the resource will never be depleted regardless the
resolution of actions' outcomes. In the remainder of the section we fix an MDP
$\mdp = (\states, \actions, \trans, \cons, \reloads, \Ca)$.

A safe run loaded with $d$ has the following two properties: (i) it never
consumes more than $\Ca$ units of the resource between 2 consecutive visits of
reload states, and (ii) it consumes at most $d$ units of the resource (energy)
before it reaches the first reload state. To ensure (i), we need to identify a
maximal subset $\reloads'\subseteq \reloads$ of reload states for which there
is a strategy $\sigma$ that, starting in some $r\in \reloads'$, can always
reach $\reloads'$ again using at most $\Ca$ resource units. To ensure (ii), we
need a strategy that suitably navigates towards $\reloads'$ while not
reloading and while using at most $d$ units of resource.

In summary, for both properties (i) and (ii) we need to find a strategy that
can surely reach a set of states ($\reloads'$) without reloading and withing a
certain limit on consumption ($\Ca$ and $d$, respectively). We capture the
desired behavior of the strategies by a new objective $\nonreloading$
(\emph{non-reloading reachability}).

\subsection{Non-reloading reachability}%
\label{sec:non-reloading}%

{The problem of non-reloading reachability in CMDPs
is similar to the problem of \emph{minimum cost reachability} on regular
MDPs with non-negative costs, which was studied
before~\cite{KBBEGRZ:short-path-interdiction}. In this sub-section, we present a new iterative algorithm for this problem which fits better into our framework and is implemented in our tool.} The reachability objective
$\reachability$ is defined as a subset of $\safety$, and thus relies on
resource levels. The following definition of $\nonreloading$ follows similar
ideas as we used for $\reachability$, but (a) ignores what happens after the
first visit of the target set, and (b) it uses the cumulative consumption
instead of resource levels to ignore the effect of reload states.

Given $\target\subseteq \states$ and $i\in\Nset$, the objective
$\nonreloading[\target]^i$ (\emph{bounded non-reloading reachability}) is the
set of all (not necessary safe) loaded runs
$\loadedpath{s}{d}_1a_1s_2a_2\ldots$ such that for some $1 \leq f \leq i+1$ it
holds $s_f\in\target$ and $\sum_{j=1}^{f-1}\cons(s_j, a_j) \leq d$. The union
$\nonreloading[\target]=\bigcup_{i\in\Nset}\nonreloading[\target]^i$ forms the
\emph{non-reloading reachability} objective.

Let us now fix some $\target\subseteq \states$. In the next few paragraphs, we
discuss the solution of sure non-reloading reachability of $\target$:
computing the vector $\ml{\nonreloading[\target]}{}$ and the corresponding
witness strategy. The solution is based on backward induction (with respect to
number of steps needed to reach $\target$). The key concept here is the
\emph{value of action} $a$ in a state $s$ based on a vector
$\vec{v}\in\extNset^{\punct{\states}}$, denoted as $\AV{\vec{v}}{s}{a}$ and
defined as follows.
\[
\AV{\vec{v}}{s}{a} = \cons(s,a) + \max_{t\in \Succ(s,a)} \veccomp{\vec{v}}{t}
\]
Intuitively, $\AV{\vec{v}}{s}{a}$ is the consumption of $a$ in $s$ plus the
worst value of $\vec{v}$ among the relevant successors. Now imagine that
$\vec{v}$ is equal to $\ml{\nonreloading^{i}}{}$; that is, for each
state $s$ it contains the minimal amount of resource needed (without reloading)
to reach $\target$ in at most $i$ steps. Then, $\AVname$ for $a$ in $s$ is the
minimal amount of resource needed to reach $\target$ in $i+1$ steps when playing
$a$ in $s$.

The following functional $\mathcal{F} \colon \extNset^\states \to
\extNset^\states$ is a simple generalization of the standard Bellman
functional used for computing shortest paths in graphs. We use
$\mathcal{F}^i(\vec{v})$ for the result of $i$ applications of $\mathcal{F}$
on $\vec{v}$.
\[
\veccomp{\mathcal{F}(\vec{v})}{s} =
\begin{cases}
  0 & s \in \target\\
  \min_{a\in \actions} \AV{\vec{v}}{s}{a} & s \not \in \target\\
\end{cases}
\]
To complete our induction-based computation, we need to find the right
initialization vector $\vec{x}_\target$ for $\mathcal{F}$. As the intuition
for action value hints, $\vec{x}_\target$ should be precisely
$\ml{\nonreloading^0}{}$ and thus is defined as
$\veccomp{\vec{x}_\target}{s}=0$ for $s\in\target$ and as
$\veccomp{\vec{x}_\target}{s}=\infty$ otherwise.

\begin{lemma}
\label{lem:stepreach-cost} It holds that
$\ml{\nonreloading[\target]^i}{} = \mathcal{F}^i(\vec{x}_T)$
for every $i\geq 0$.
\end{lemma}

\begin{proof}
We proceed by induction on $i$. The base case for $i=0$ is trivial. Now assume
that the lemma holds for some $i\geq 0$.

From the definition of $\nonreloading^i$ we have that a loaded run $\loadedpath{\run}{d}=\loaded{s}{d}_1a_1s_2\ldots$ satisfies
surely $\nonreloading^{i+1}$ if and only if $s_1 \in \target$ or
$\loadedpath{\run}{h}\suff{2}\in\nonreloading^i$ for $h=d-\cons(s_1,a_1)$.
Therefore, given a state $s\notin \target$ and the load
$d=\veccomp{\ml{\nonreloading^{i+1}}{}}{s}$, each witness strategy $\sigma$
for $\ml{\nonreloading^{i+1}}{}$ must guarantee that if
$\sigma(\loadedpath{s}{d})=a$ then $d \geq \veccomp{\ml{\nonreloading^i}{}}{s'} +
\cons(s,a)$ for all $s'\in \Succ(s,a)$. That is,
$d \geq \AV{\ml{\nonreloading^i}{}}{s}{a} = \AV{\mathcal{F}^i(\vec{x}_T)}{s}{a}$.

On the other hand, let $a_m$ be the action with minimal $AV$ for $s$ based on
$\ml{\nonreloading^i}{}$. The strategy that plays $a_m$ in the first step and then mimics some witness strategy for $\ml{\nonreloading^i}{}$
surely satisfies $\nonreloading^{i+1}$ from $s$ loaded by $\AV{\ml{\nonreloading^i}{}}{s}{a_m}$. Therefore, $d \leq
\AV{\ml{\nonreloading^i}{}}{s}{a_m} = \AV{\mathcal{F}^i(\vec{x}_T)}{s}{a_m}$.

Together, $d = \veccomp{\ml{\nonreloading^{i+1}}{}}{s} =
\min_{a\in A}\AV{\ml{\nonreloading^i}{}}{s}{a} =
\veccomp{\mathcal{F}(\ml{\nonreloading^i}{})}{s}$ and that is by induction
hypothesis equal to $\veccomp{\mathcal{F}(\mathcal{F}^{i}(\vec{x}_\target))}{s}=\veccomp{\mathcal{F}^{i+1}(\vec{x}_\target)}{s}$.
\end{proof}

\begin{theorem}
\label{thm:minreach-main} Denote by $n$ the length of the longest simple path
in $\mdp$. Iterating $\mathcal{F}$ on $\vec{x}_T$ yields a fixed point in at most
$n$ steps and this fixed point equals $\ml{\nonreloading[\target]}{}$.
\end{theorem}

\begin{proof}
For the sake of contradiction, suppose that $\mathcal{F}$ does not yield a
fixed point after $n$ steps. Then there exists a state $s$ such that $d =
\veccomp{\mathcal{F}^{n+1}(\vec{x}_\target)}{s} <
\veccomp{\mathcal{F}^{n}(\vec{x}_\target)}{s}$. Let $\sigma$ be a witness
strategy for $\mathcal{F}^{n+1}(\vec{x}_\target) =
\ml{\nonreloading^{n+1}}{}$. Now let $\loadedpath{\run}{d} =
\loadedpath{s}{d}_1a_1s_2\ldots$ be a loaded run from
$\lcompatible{\sigma}{s}{d}$ such that $s_i\notin \target$ for all $i \leq
n+1$ and $s_{n+2}\in \target$, and such that for all $1 \leq k \leq n+1$ it
holds $d-c_k = \veccomp{\ml{\nonreloading^{n+1-k}}{}}{s_{k+1}}$ where $c_k =
\sum_{j=1}^{k} \cons(s_{j}, a_{j})$ is the consumption of the first $k$
actions of $\run$. Such a run must exist, otherwise some
$\ml{\nonreloading^i}{}$ can be improved.

As $n$ is the length of the longest simple path in $\mdp$, we can conclude
that there are two indices $f < l \leq n+1$ such that $s_f = s_l = t$. But
since $\mdp$ is decreasing, we have that $c_f < c_l$ and thus
$\veccomp{\ml{\nonreloading^{n+1-f}}{}}{t} = d - c_f > d - c_l =
\veccomp{\ml{\nonreloading^{n+1-l}}{}}{t}$. As $n+1-f > n+1-l$, we reached a
contradiction with the fact that $\nonreloading^{n+1-f} \supseteq
\nonreloading^{n+1-l}$.

By \cref{lem:stepreach-cost} we have that $\mathcal{F}^n(\vec{x}_\target) =
\ml{\nonreloading}{}$.
\end{proof}

\textbf{Witness strategy for $\ml{\nonreloading}{}$.} Any memoryless strategy
$\sigma$ that picks for each history ending with a state $s$ some action $a_s$
such that $\AV{\ml{\nonreloading}{}}{s}{a_s} =
\veccomp{\ml{\nonreloading}{}}{s}$ is clearly a witness strategy for
$\ml{\nonreloading}{}$, which is,
$\vecsat{\sigma}{\ml{\nonreloading}{}}{\nonreloading}{}$.

\subsection{Safely reaching reloads from reloads}%
The objective $\nonreloading$ is sufficient for the property (ii) with
$T=\reloads'$. But we cannot use it off-the-shelf to guarantee the property
(i) at most $\Ca$ units of resource are consumed between two consecutive
visits of $\reloads\punct{'}$. For that, we need to solve the problem of
\emph{reachability within at least $1$ steps} (starting in $\target$ alone
does not count as reaching $\target$ here). We define
$\nonreloading[+\target]^i$ in the same way as $\nonreloading[\target]$ but we
enforce that $f > 1$ and we set
$\nonreloading[+\target]=\bigcup_{i\in\Nset}\nonreloading[+\target]$. To
compute $\ml{\nonreloading[+\target]}{}$, we slightly alter $\mathcal{F}$
using the following truncation operator.

\[
\veccomp{\strunc[\target]{\vec{v}}}{s} =
\begin{cases}
  \veccomp{\vec{v}}{s} &
    \text{if } s \not \in \target, \\
  0 & \text{if }s \in \target.\\
\end{cases}\]

The new functional $\mathcal{G}$ applied to $\vec{v}$ computes the new value
in the same way for all states (including states from $\target$), but treats
$\veccomp{\vec{v}}{t}$ as $0$ for $t \in \target$.

\[
\veccomp{\mathcal{G}(\vec{v})}{s} =
\min\nolimits_{a\in \actions} \AV{\strunc[\target]{\vec{v}}}{s}{a}
\]

Let $\infvec^\states\in \extNset^{\states}$ denote the vector with all components equal to
$\infty$. Clearly, $\strunc[\target]{\infvec^\states} = \vec{x}_\target$, and thus it
is easy to see that for all $s\notin\target$ and $i\in\Nset$ we have that
$\veccomp{\mathcal{G}^i(\infvec^\states)}{s} =
\veccomp{\mathcal{F}^i(\vec{x}_\target)}{s}$. Moreover,
$\strunc[\target]{\mathcal{G}^i(\infvec^\states)} = \mathcal{F}^i(\vec{x}_\target)$. A
slight modification of arguments used to prove \cref{lem:stepreach-cost} and
\cref{thm:minreach-main} shows that $\mathcal{G}$ indeed computes
$\ml{\nonreloading[+\target]}{}$ and we need at most $n+1$ iterations for the
desired fixed point. \Cref{algo:mininitcons_iterative} iteratively applies $\mathcal{G}$ on $\infvec^\states$ until a fixed point is reached.

\begin{algorithm}[ht]
\KwIn{CMDP $\mdp=(\states, \actions, \trans, \cons, \reloads, \Ca)$ and $\target\subseteq \states$}
\KwOut{The vector $\ml[\mdp]{\nonreloading[+\target]}{}$}
$\vec{v} \leftarrow \infvec^\states$\;
\Repeat{$\,\vec{v}_{\mathit{old}} = \vec{v}$}{
  $ \vec{v}_\old \leftarrow \vec{v} $\;
  \ForEach{$ s \in \states $}{
    $c \leftarrow
      \min_{a\in \actions}\AV{\strunc[\target]{\vec{v}_\old}}{s}{a}$\;
    \If{$c < \vec{v}(s)$}{
      $ \vec{v}(s) \leftarrow c$\;
    }
  }
}
\Return{$\vec{v}$}
\caption{Computing $\ml{\nonreloading[+\target]}{}$.}
\label{algo:mininitcons_iterative}
\end{algorithm}

\begin{theorem}
\label{thm:mininitconst-main} Given a CMDP $\mdp$ and a set of target states
$\target$, \cref{algo:mininitcons_iterative} computes the vector
$\ml[\mdp]{\nonreloading[+\target]}{}$. Moreover, the repeat-loop terminates
after at most $|\states|$ iterations.
\end{theorem}

\begin{proof}
Each iteration of the repeat-loop computes an application of $\mathcal{G}$ on
the value of $\vec{v}$ from line $3$ (stored in $\vec{v}_\old$) and stores the
resulting values in $\vec{v}$ on line $7$. Iterating $\mathcal{G}$ on
$\infvec^\states$ yields a fixed point in at most $n+1$ iterations where $n$ is
the length of the longest simple path in $\mdp$. As $n+1 \leq |\states|$, the
test on line 8 becomes true after at most $|\states|$ iterations and $\vec{v}$
on line $8$ contains the result of $\mathcal{G}^i(\infvec^\states)$ where $i$
is the actual number of iterations. Thus, the value of $\vec{v}$ on line $9$
is equal to $\ml[\mdp]{\nonreloading[+\target]}{}$ and is computed in at most
$|\states|$ iterations.
\end{proof}

Now with \cref{algo:mininitcons_iterative} we can compute
$\ml[\mdp]{\nonreloading[+\reloads]}{}$ and see which reload states should be
avoided by safe runs: the reloads that need more than $\Ca$ units
of resource to surely reach $\reloads$ again. We call such reload states
\emph{unusable in $\mdp$}.	w

\subsection{Detecting useful reloads and solving the safety problem}

Using \cref{algo:mininitcons_iterative}, we can identify reload states that
are unusable in $\mdp$. However, it does not automatically mean that the rest
of the reload states form the desired set $\reloads'$ for property (i).
Consider the CMDP $\mathcal{D}$ in \cref{fig:nonreloading-not-enough}. The
only reload state that is unusable is $w$
($\veccomp{\ml{\nonreloading[+\reloads]}{}}{w} = \infty$). But clearly, all
runs that avoid $w$ must avoid $v$ and $x$ as well. This intuition is backed
up by the fact that $\veccomp{\ml{\nonreloading[+\reloads_1]}{}}{v} =
\veccomp{\ml{\nonreloading[+\reloads_1]}{}}{x} = \infty$ for $\reloads_1 =
\reloads \smallsetminus \{w\} = \{u, v, x\}$, see
\cref{fig:nonreloading-not-enough-ii}. The property (i) indeed translates to
$\veccomp{\ml{\nonreloading[+\reloads']}{}}{r} \leq \Ca$ for all
$r\in\reloads'$; naturally, we want to identify the maximal
$\reloads'\subseteq \reloads$ for which this holds. \Cref{algo:safety} finds
the desired $\reloads'$ by iteratively removing unusable
reloads from the current candidate set $\reloads'$ until there is no unusable
reload in $\reloads'$ (lines 3-7).

\begin{figure}[ht]
\begin{center}
\begin{tikzpicture}[automaton]
\begin{scope}[every node/.append style={state}]
\node (s0) {$t$};
\node (s1) [reload, right of=s0] {$u$};
\node (s2) [reload, right of=s1] {$v$};
\node (s3) [reload, right of=s2] {$w$};
\node (s4) [reload, above right=.2 and 1.3 of s3] {$x$};
\node (s5) [below right=.2 and 1.3 of s3] {$y$};
\end{scope}

\begin{scope}[every node/.append style={valbox}]
\node[nonreloading] at (s0) {$3$};
\node[nonreloading] at (s1) {$1$};
\node[nonreloading] at (s2) {$1$};
\node[nonreloading] at (s3) {$\infty$};
\node[nonreloading, above] at (s4) {$2$};
\node[nonreloading] at (s5) {$\infty$};
\end{scope}

\path[->, auto]
(s0) edge[bend right] node[cons, swap] {$3$} (s1)
(s1) edge[bend right] node[cons, swap] {$2$} (s0)
(s1) edge node[cons] {1} (s2)
(s2) edge node[cons] {1} (s3)
(s3) edge[out=0, in=240, looseness=.9]
  node[above left, pos=.8, prob] {$\frac{9}{10}$}
(s4.240)
(s3) edge[out=0, in=120, looseness=.9]
  node[below left, pos=.8, prob] {$\frac{1}{10}$}
  node[cons, swap] {1}
(s5.120)
(s4) edge[out=160, in=60] node[cons, above] {2} (s3)
(s5) edge[loop right] node[cons] {1} (s5)
;

\node[capacity, below of=s1, node distance=1.3cm, overlay] {10};

\end{tikzpicture}
\end{center}
\caption{A CMDP $\mathcal{D}$ with values of $\ml{\nonreloading[+\reloads]}{}$. For a state $s$, the value $\veccomp{\ml{\nonreloading[+\reloads]}{}}{s}$ is pictured in the orange box below $s$.}
\label{fig:nonreloading-not-enough}
\end{figure}
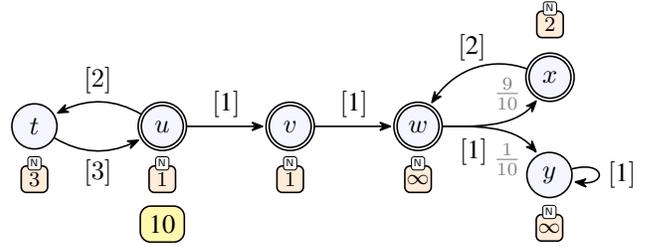

\begin{figure}[ht]
\begin{center}
\begin{tikzpicture}[automaton]
\begin{scope}[every node/.append style={state}]
\node (s0) {$t$};
\node (s1) [reload, right of=s0] {$u$};
\node (s2) [reload, right of=s1] {$v$};
\node (s3) [right of=s2] {$w$};
\node (s4) [reload, above right=.2 and 1.3 of s3] {$x$};
\node (s5) [below right=.2 and 1.3 of s3] {$y$};
\end{scope}

\begin{scope}[every node/.append style={valbox}]
\node[nonreloading] at (s0) {$3$};
\node[nonreloading] at (s1) {$5$};
\node[nonreloading] at (s2) {$\infty$};
\node[nonreloading] at (s3) {$\infty$};
\node[nonreloading, above] at (s4) {$\infty$};
\node[nonreloading] at (s5) {$\infty$};
\end{scope}

\path[->, auto]
(s0) edge[bend right] node[cons, swap] {$3$} (s1)
(s1) edge[bend right] node[cons, swap] {$2$} (s0)
(s1) edge node[cons] {1} (s2)
(s2) edge node[cons] {1} (s3)
(s3) edge[out=0, in=240, looseness=.9]
  node[above left, pos=.8, prob] {$\frac{9}{10}$}
(s4.240)
(s3) edge[out=0, in=120, looseness=.9]
  node[below left, pos=.8, prob] {$\frac{1}{10}$}
  node[cons, swap] {1}
(s5.120)
(s4) edge[out=160, in=60] node[cons, above] {2} (s3)
(s5) edge[loop right] node[cons] {1} (s5)
;

\node[capacity, below of=s1, node distance=1.3cm, overlay] {10};
\end{tikzpicture}
\end{center}
\caption{The CMDP $\mathcal{D}(\{u, v, x\})$ with values of
$\ml{\nonreloading[+\{u, v, x\}]}{}$ in orange boxes. Note the difference in the set of reload states in comparison to \cref{fig:nonreloading-not-enough}.}
\label{fig:nonreloading-not-enough-ii}
\end{figure}
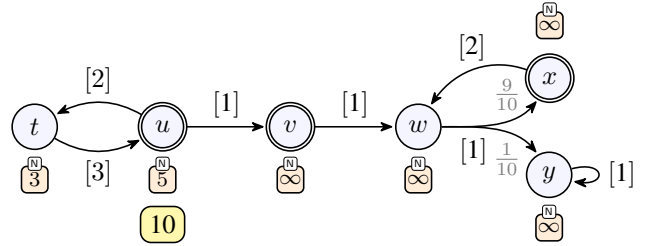

With the right set $\reloads'$ in hand, we can move on to the property (ii) of
safe runs: navigating safely towards reloads in $\reloads'$, which equals to
the objective $\nonreloading[\reloads']$ from \cref{sec:non-reloading}. We can
reuse \cref{algo:mininitcons_iterative} for it as $\ml{\nonreloading}{} =
\strunc{\ml{\nonreloading[+]}{}}$ regardless the target set. Based on
properties (i) and (ii), we claim that $\ml{\safety}{} =
\strunc[\reloads']{\ml{\nonreloading[+\reloads']}{}}$. Indeed, we need at most
$\Ca$ units of resource to move between reloads of $\reloads'$, and we need at
most $\veccomp{\strunc[\reloads']{\ml{\nonreloading[+\reloads']}{}}}{s}$ units
of resource to reach $\reloads'$ from $s$.

Whenever $\veccomp{\ml{\safety}{}}{s} > \Ca$ for some state $s$, the exact
value is not important for us; the meaning is still that there is no strategy
$\sigma$ and no initial load $d\leq \Ca$ such that
$\sdsat{\sigma}{s}{d}{\safety}{}$. Hence, we can set
$\veccomp{\ml{\safety}{}}{s} = \infty$ in all such cases. To achieve this, we
extend the operator $\strunc[\target]{\cdot}$ into
$\trunc[\target]{\cdot}^\Ca$ as follows.
\[
\veccomp{\trunc[\target]{\vec{x}}^\Ca}{s} =
\begin{cases}
  \infty &
    \text{if }\veccomp{\vec{x}}{s} > \Ca\\
  \vec{x}(s) &
    \text{if }
    \veccomp{\vec{x}}{s} \leq \Ca\text{ and } s \not \in \target\\
  0 & \text{if }\veccomp{\vec{x}}{s} \leq \Ca\text{ and }s \in \target\\
\end{cases}
\]

\begin{algorithm}[ht]
\KwIn{CMDP $\mdp=(\states, \actions, \trans, \cons, \reloads, \Ca)$}
\KwOut{The vector $\ml[\mdp]{\safety}{}$}
$\varRel \leftarrow \reloads $\;
$\varToRemove \leftarrow \emptyset$\;
\Repeat{$\varToRemove = \emptyset$}{
  $ \varRel \leftarrow \varRel\smallsetminus \varToRemove $\;
  $ \vec{n} \leftarrow \ml[\mdp]{\nonreloading[+\varRel]}{}$\label{aline:mcon}\tcc*[r]{\scriptsize\Cref{algo:mininitcons_iterative} with $\target = \varRel$}
  $ \varToRemove \leftarrow
    \{ r \in \varRel \mid \veccomp{\vec{n}}{r} > \Ca \}$\;
}
\Return{$\trunc[\varRel]{\vec{n}}^\Ca$}\;
\caption{Computing $\ml{\safety}{}$.}
\label{algo:safety}
\end{algorithm}

\begin{theorem}%
\label{thm:safety-main}%
Algorithm~\ref{algo:safety} computes the vector $\ml[\mdp]{\safety}{}$ in
time polynomial with respect to the representation of $\mdp$.
\end{theorem}
\begin{proof}
\textbf{Complexity.} The algorithm clearly terminates. Computing
$\ml{\nonreloading[+\reloads']}{}$ on line~\ref{aline:mcon} takes a
polynomial number of steps per call (\cref{thm:mininitconst-main}). Since the
repeat loop performs at most $|\reloads|$ iterations, the complexity follows.

\textbf{Correctness.} We first prove that upon termination
$\veccomp{\trunc[\varRel]{\vec{n}}^\Ca}{s} \leq
\veccomp{\ml[\mdp]{\safety}{}}{s}$ for each $s\in\states$ whenever the latter
value is finite. This is implied by the fact that
$\ml{\nonreloading[+\varRel]}{} \leq \ml{\safety}{}$ is the invariant of the
algorithm. To see that, it suffices to show that at every point of execution,
$\veccomp{\ml{\safety}{}}{t} = \infty $ for each $ t\in \reloads\smallsetminus
\varRel$: if this holds, each strategy that satisfies $\safety$ must avoid
states in $\reloads \smallsetminus \varRel$ (due to property (i) of safe runs)
and thus the first reload on runs compatible with such a strategy must be from
$\varRel$.

Let $\varRel_i$ denote the contents of $\varRel$ after the $i$-th iteration.
We prove, by induction on $i$, that $\veccomp{\ml{\safety}{}}{t} = \infty$ for
all $t\in \reloads \smallsetminus \varRel$. For $ i = 0 $ we have $\reloads =
\varRel_0$, so the statement holds. For $i>0$ , let $t \in \reloads
\smallsetminus \varRel_{i}$, then it must exist some $j<i$ such that
$\veccomp{\vec{n}}{t} = \veccomp{\ml{\nonreloading[+\varRel_{j}]}{}}{t} >
\Ca$, hence no strategy can safely reach $\varRel_j$ from $t$ and by induction
hypothesis, the reload states from $\reloads \smallsetminus \varRel_j$ must be
avoided by strategies that satisfy $\safety$. Together, as $\mdp$ is
decreasing, there is no strategy $\sigma$ such that
$\sdsat{\sigma}{t}{\Ca}{\safety}{}$ and hence $\veccomp{\ml{\safety}{}}{t} =
\infty$.

Finally, we need to prove that upon termination, $\trunc[\varRel]{\vec{n}}^\Ca
\geq \ml{\safety}{}$. As $\vec{n} = \ml{\nonreloading[+\varRel]}{}$ and
$\veccomp{\vec{n}}{r} \leq \Ca$ for each $r\in \varRel$, then, for each $s$
with $d = \veccomp{\trunc[\varRel]{\vec{n}}^\Ca}{s} \leq \Ca$ there exists a
strategy that can reach $\varRel \subseteq \reloads$ consuming at most $d$
units of resource, and, once in $\varRel$, $\sigma$ can always return to
$\varRel$ within $\Ca$ units of resource. Thus, all runs in
$\lcompatible{\sigma}{s}{d}$ are safe and $d$ is enough for $\safety$ in
$s$.
\end{proof}

\subsection{Safe strategies}

\begin{definition}
\label{def:safeact} Let $s\in\states$ be a state and let $0\leq d\leq\Ca$ be a
resource level. We call an action $a$ \emph{safe} in $s$ with $d$ if (1)
$v=\AV{\ml{\safety}{}}{s}{a} \leq d$, or if (2) $v\leq \Ca$ and
$s\in\reloads$, or if (3) $\veccomp{\ml{\safety}{}}{s} > \Ca$. Further, $a$ is
\emph{\minsafe{}} in $s$ if it is safe in $s$ with
$d=\veccomp{\ml{\safety}{}}{s}$. We call a strategy $\sigma$ \emph{safe} if it
picks an action that is safe in the current state with the current resource
level whenever possible.
\end{definition}

\textbf{Remarks.} By definition, no action is safe in $s$ for all $r <
\veccomp{\ml{\safety}{}}{s} < \infty$ (otherwise $\veccomp{\ml{\safety}{}}{s}$
is at most $r$). On the other hand, there is always at least one action that
is \minsafe{} for each state $s$ and, in particular, all actions are safe and
\minsafe{} in $s$ with $\veccomp{\ml{\safety}{}}{s} = \infty$.

\begin{lemma}%
\label{lem:safe-act}%
Let $\sigma$ be a safe strategy. Then
$\vecsat{\sigma}{\ml{\safety}{}}{\safety}{}$.
\end{lemma}

\begin{proof}
We need to prove that, given a state $s$ with $\veccomp{\ml{\safety}{}}{s}
\leq \Ca$ and an initial load $d$ such that $\veccomp{\ml{\safety}{}}{s} \leq
d \leq \Ca$, all runs in $\lcompatible{\sigma}{s}{d}$ are safe. To do this, we
show that all $s$-initiated $d$-loaded paths $\loadedpath{\apath}{d}$ created
by $\sigma$ are safe. By simple induction with respect to the length of the
path we prove that $\lreslev{d}{\apath} \geq \veccomp{\ml{\safety}{}}	{t} \neq
\bot$ where $t = \last{\apath}$. For $\loaded{s}{d}$ this clearly holds. Now
assume that $\lreslev{d}{\apath} \geq \veccomp{\ml{\safety}{}}{t} \neq \bot$
for some $s$-initiated $\apath$ with $t=\last{\apath}$ and all $d \geq
\veccomp{\ml{\safety}{}}{s}$. Now consider a $d'$-loaded path
$\loadedpath{s}{d'}'as\histconc\apath$ created by $\sigma$; we have
$\lreslev{d'}{s'as} = d' - \cons(s', a) \geq \veccomp{\ml{\safety}{}}{s}$ by
definition of action value and safe actions, and thus by the induction
hypothesis we have that $\lreslev{d'}{s'as\histconc\apath} \geq
\veccomp{\ml{\safety}{}}{t}$.
\end{proof}

\begin{theorem}
\label{thm:safety-strat}%
In each consumption MDP $\mdp$ there is a memoryless strategy $\sigma$ such
that $\vecsat[\mdp]{\sigma}{\ml[\mdp]{\safety}{}}{\safety}{}$. Moreover,
$\sigma$ can be computed in time polynomial with respect to the representation
of $\mdp$.
\end{theorem}

\begin{proof}
Using \cref{lem:safe-act}, the existence of a memoryless strategy
follows from the fact that a strategy that fixes one \minsafe{} action in each
state is safe. The complexity follows from \cref{thm:safety-main}.
\end{proof}

\begin{example}\label{ex:safety}
\Cref{fig:example+vectors} shows again the CMDP from \cref{fig:example} and
includes also values computed by
\cref{algo:mininitcons_iterative,algo:safety}. \Cref{algo:safety} stores the
values of $\ml{\nonreloading}{}$ into $\vec{n}$ and, because no value is
$\infty$, returns just $\trunc{\vec{n}}_{\reloads}^{\Ca}$. The strategy
$\sigma_{\hacta}$ from \cref{ex:resource_strategy} is a witness strategy for $\ml{\safety}{}$. As
$\veccomp{\ml{\safety}{}}{s}=2$, no strategy would be safe from $s$ with
initial load $1$.
\end{example}

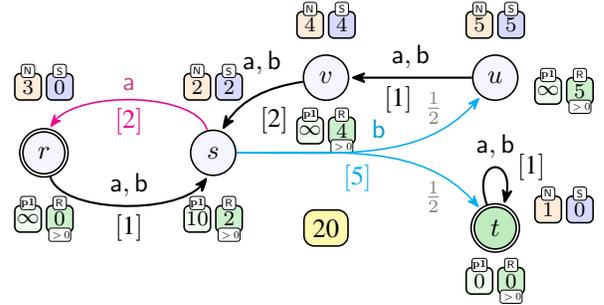
\begin{figure}[h]
\centering
\begin{tikzpicture}[automaton, xscale=1.5]
\tikzstyle{pvect} = [fill=green!70!black!5]
\tikzstyle{a} = [magenta]
\tikzstyle{b} = [cyan]
\tikzstyle{ab} = [black, thick]
\node[state] (start) at (0,0)  {$s$};
\node[state,reload,target] (t)   at (2.5,-1) {$t$};
\node[state,reload] (rel)   at (-1.5,0) {$r$};
\node[state] (mis)      at (2.5,1)  {$u$};
\node[state] (pathhome) at (1,1)  {$v$};

\begin{scope}[every node/.append style={valbox, above}]
\node[nonreloading, xshift=-6pt, yshift=5pt] at (start) {$2$};
\node[nonreloading, xshift=0pt, right, yshift=7pt] at (t) {$1$};
\node[nonreloading, xshift=-6pt, yshift=5pt] at (rel) {$3$};
\node[nonreloading, xshift=-6pt] at (mis) {$5$};
\node[nonreloading, xshift=-6pt] at (pathhome) {$4$};
\end{scope}

\begin{scope}[every node/.append style={valbox, above}]
\node[safety, xshift=6pt, yshift=5pt] at (start) {$2$};
\node[safety, xshift=12pt, right, yshift=7pt] at (t) {$0$};
\node[safety, xshift=6pt, yshift=5pt] at (rel) {$0$};
\node[safety, xshift=6pt] at (mis) {$5$};
\node[safety, xshift=6pt] at (pathhome) {$4$};
\end{scope}

\begin{scope}[every node/.append style={valbox, below}]
\node[pvect, objective={$\vec{p1}$}, xshift=-6pt, yshift=-5pt] at (start) {$10$};
\node[pvect, objective={$\vec{p1}$}, xshift=-6pt] at (t) {$0$};
\node[pvect, objective={$\vec{p1}$}, xshift=-6pt, yshift=-5pt] at (rel) {$\infty$};
\node[pvect, objective={$\vec{p1}$}, xshift= 0pt, right, yshift=-5pt] at (mis) {$\infty$};
\node[pvect, objective={$\vec{p1}$}, xshift=-6pt] at (pathhome) {$\infty$};
\end{scope}

\begin{scope}[every node/.append style={valbox, below}]
\node[pos reach, xshift=6pt, yshift=-5pt] at (start) {$2$};
\node[pos reach, xshift=6pt] at (t) {$0$};
\node[pos reach, xshift=6pt, yshift=-5pt] at (rel) {$0$};
\node[pos reach, xshift=12pt, right, yshift=-5pt] at (mis) {$5$};
\node[pos reach, xshift=6pt] at (pathhome) {$4$};
\end{scope}

\path[->,auto,swap]
(start) edge[bend right=80, a]
  node[] {$\hacta$}
  node[cons, swap] {2}
(rel)
(rel) edge[bend right=80, ab]
  node[swap] {$\acta,\actb$}
  node[cons] {1}
(start)
(start)
  edge[out=0, in=120,looseness=1.1, b]
    node[cons] {5}
    node[prob,below left,pos=.8] {$\frac{1}{2}$}
  (t)
  edge[out=0, in=240,looseness=1.1, b]
    node[swap, pos=.56] {$\hactb$}
    node[prob,above left,pos=.8] {$\frac{1}{2}$}
  (mis)
(mis) edge[ab]
  node {$\acta,\actb$}
  node[cons,swap, pos=.6] {1}
(pathhome)
(pathhome) edge[ab, bend right]
  node[pos=.1] {$\acta,\actb$}
  node[cons, swap, pos=.5] {2}
(start)
(t) edge[loop above, ab, looseness=16]
  node[] {$\acta,\actb$}
  node[cons, right, outer sep=5pt] {1}
(t)
;

\node[capacity] at (1, -1) {20};
\end{tikzpicture}
\caption{The CMDP from \cref{fig:example} with vectors for \cref{ex:safety,ex:posreach-iter}. The values of
$\ml{\nonreloading[\reloads]}{}$ and $\ml{\safety}{}$ (referenced in \cref{ex:safety}) are indicated by the orange
(left) and blue (right) boxes above states, respectively. Values of $\vec{p}1$ (left) and $\ml{\reachability}{\pos}$ (right) that are referenced in \cref{ex:posreach-iter} are indicated by the green boxes below states.}%
\label{fig:example+vectors} \end{figure}
\section{Positive reachability}%
\label{sec:posreach} In this section, we present the solution of the problem
called positive reachability. We focus on strategies that, given a set
$\target \subseteq \states$ of target states, safely satisfy
$\reachability[\target] \subseteq \safety$ with positive probability. The main
contribution of this section is Algorithm~\ref{algo:posreach} that computes
$\ml{\reachability}{\pos}$ and the corresponding witness strategy. As before,
for the rest of this section we fix a CMDP $\mdp=(\states, \actions, \trans,
\cons, \reloads, \Ca)$ and also a set $\target \subseteq \states$.

Let $s\in\states\smallsetminus\target$ be a state, let $d$ be an initial load
and let $\sigma$ be a strategy such that
$\sdsat{\sigma}{s}{d}{\reachability[\target]}{\pos}$. Intuitively, as
$\sdsat{\sigma}{s}{d}{\reachability}{\pos}$ implies
$\sdsat{\sigma}{s}{d}{\safety}{}$, the strategy is limited to safe actions.
For the reachability part, $\sigma$ must start with an action
$a=\sigma(\loadedpath{s}{d})$ such that for at least successor $s'$ of $a$ in
$s$ it holds that $\sdsat{\sigma}{s'}{d'}{\reachability}{\pos}$ with $d' = d -
\cons(s,a)$. It must then continue in a similar fashion from $s'$ until either
$\target$ is reached (and $\sigma$ produces the desired run from
$\reachability[\target]$) or until there is no such action.

To formalize the intuition, we define two auxiliary functions. Let us fix a
state $s$, and action $a$, a successor $s'\in\Succ(s,a)$, and a vector
$\vec{x}\in\extNset^\states$. We define the \emph{hope value} of
$s'$ for $a$ in $s$ based on $\vec{x}$, denoted by $\shope{\vec{x}}{s}{a}{s'}$, and the \emph{safe value} of $a$ in $s$ based on $\vec{x}$, denoted by $\SV{\vec{x}}{s}{a}$, as follows.

\[
\begin{aligned}
\shope{\vec{x}}{s}{a}{s'} &= %
\max_{\substack{%
  t \in \Succ(s,a)\\
  t \neq s'
}} \{\veccomp{\vec{x}}{s'}, \veccomp{\ml{\safety}{}}{t}\}\\
\SV{\vec{x}}{s}{a} &= \cons(s, a) + \min_{s'\in\Succ(s, a)} \shope{\vec{x}}{s}{a}{s'}
\end{aligned}
\]

The hope value of $s'$ for $a$ in $s$ represents the lowest level of resource
that the agent needs to have after playing $a$ in order to (i) have at least
$\veccomp{\vec{x}}{s'}$ units of resource if the outcome of $a$ is $s'$, and
(ii) to survive otherwise.

We again use functionals to iteratively compute $\ml{\reachability^i}{\pos}$,
with a fixed point equal to $\ml{\reachability}{\pos}$. The main operator
$\mathcal{B}$ just applies $\trunc{\cdot}_\reloads^\Ca$ on the result of the
auxiliary functional $\mathcal{A}$. The application of
$\trunc{\cdot}_\reloads^\Ca$ ensures that whenever the result is higher than
$\Ca$, it set to $\infty$, and that in reload states the value is either $0$
or $\infty$, which is in line what is expected from $\ml{\reachability}{\pos}$.

\begin{align*}
\veccomp{\mathcal{A}(\vec{x})}{s} &=
  \begin{dcases}
    \veccomp{\ml{\safety}{}}{s} & \text{if } s \in \target \\
    \min_{a\in A} \SV{\vec{x}}{s}{a}&
      \text{otherwise};
  \end{dcases}\\
\mathcal{B}(\vec{x}) &= \trunc{\mathcal{A}(\vec{x})}_\reloads^\Ca
\end{align*}

By $\vec{y}_\target$ we denote a vector such that
\[
\veccomp{\vec{y}_\target}{s} =
\begin{cases}
  \veccomp{\ml{\safety}{}}{s} & \text{if } s \in \target \\
  \infty & \text{if } s \not\in \target.
\end{cases}
\]

The following two lemmata relate $\mathcal{B}$ to $\ml{\reachability}{\pos}$
and show that $\mathcal{B}$ applied iteratively on $\vec{y}_\target$ reaches
fixed point in a number of iterations that is polynomial with respect to the representation of $\mdp$. Their proofs are quite
straightforward but technical, hence we moved them to Appendix.

\begin{lemma}
\label{lem:posreach-iterate}%
Consider the iteration of $\mathcal{B}$ on the initial vector
$\vec{y}_\target$. Then for each $ i \geq 0 $ it holds that
$\mathcal{B}^i(\vec{y}_\target) = \ml{\reachability^i}{\pos}$.
\end{lemma}

\begin{lemma}%
\label{lem:posreach-bound}%
Let $ K = |\reloads| + (|\reloads|+1)\cdot(|\states|-|\reloads|+1)$. Then
$\mathcal{B}^{K}(\vec{y}_T) = \ml{\reachability}{\pos}$.
\end{lemma}

\begin{algorithm}[t]
  \KwIn{CMDP $\mdp=(\states, \actions, \trans, \cons, \reloads, \Ca)$ and
  $\target\subseteq \states$}%
  \KwOut{The vector $\ml[\mdp]{\reachability[\target]}{\pos}$, rule selector
  $\selector$}%
  compute $\ml{\safety}{}$ \tcc*[r]{\scriptsize\Cref{algo:safety}}%
  \label{aline:prinitb}%
  \ForEach{$s\in \states$\label{aline:pr-sel-init-b}}%
  {
    $\selector(s)(\veccomp{\ml{\safety}{}}{s}) \leftarrow \text{arbitrary \minsafe{} action of }s$
  }\label{aline:pr-sel-init-e}%
  $\vec{p} \leftarrow \{\infty\}^S  $\;\label{aline:pr-init-yt-i}%
  \lForEach{$t\in \target$}%
  {$\veccomp{\vec{p}}{t} \leftarrow \veccomp{\ml{\safety}{}}{t}$}%
  \label{aline:pr-init-yt-ii}%
  \Repeat{$\,\vec{p}_{\mathit{old}} = \vec{p}$}{
    \label{aline:pr-loop-b}%
    $\vec{p}_\old \leftarrow \vec{p}$\;%
    \ForEach{$s \in \states\smallsetminus \target$}{%
      $\veccomp{\vec{a}}{s} \leftarrow%
        \arg\min_{a\in \actions} \SV{\vec{p}_\old}{s}{a}$\;\label{aline:pr-act-sel}%
      $\veccomp{\vec{p}}{s} \leftarrow%
        \min_{a\in \actions} \SV{\vec{p}_\old}{s}{a}$\;\label{aline:pr-apply-A}%
    }%
    $\vec{p} \leftarrow \trunc{\vec{p}}_{\reloads}^\Ca$\;\label{aline:pr-trunc}%
    \ForEach{$ s\in \states\smallsetminus \target $\label{aline:pr-sel-up-b}}{%
      \If{$ \veccomp{\vec{p}}{s} < \veccomp{\vec{p}_{\old}}{s}$}{%
       $\selector(s)(\veccomp{\vec{p}}{s}) \leftarrow%
       \veccomp{\vec{a}}{s} $\;\label{aline:pr-insert}\label{aline:pr-sel-up-e}%
      }%
    }%
  }\label{aline:pr-loop-e}%
  \Return{$\vec{p},\selector$}\;%
  \caption{Computing $\ml{\reachability[\target]}{\pos}$ and a corresponding witness rule selector.} \label{algo:posreach}
\end{algorithm}

\Cref{algo:posreach} computes $\ml{\reachability}{\pos}$ and a corresponding
witness rule selector $\selector$. On lines~\ref{aline:pr-init-yt-i} and
\ref{aline:pr-init-yt-ii} $\vec{p}$ is initialized to $\vec{y}_\target$. The
repeat-loop on \crefrange{aline:pr-loop-b}{aline:pr-loop-e} iterates
$\mathcal{B}$ on $\vec{p}$ (and $\vec{p}_\old$) and builds the witness
selector gradually. In particular, the \cref{aline:pr-apply-A} stores the
application of $\mathcal{A}$ on $\vec{p}_\old$ in $\vec{p}$, the
\cref{aline:pr-trunc} performs $\mathcal{B}$, and the condition on
\cref{aline:pr-loop-e} checks for the fixed point. Finally,
lines~\ref{aline:pr-act-sel} and
\ref{aline:pr-sel-up-b}-\ref{aline:pr-sel-up-e} update $\selector$ accordingly.
The correctness and complexity of the algorithm are stated formally in theorems \ref{thm:posreach-values-main} ($\vec{p}=\ml{\reachability}{\pos}$) and \ref{thm:posreach-strat-main} (correctness of $\selector$).

\begin{example}%
\label{ex:posreach-iter}%
Consider again \cref{fig:example+vectors} which shows the values of
$\vec{p1}$: the vector $\vec{p}$ after one iteration of the repeat loop of
\cref{algo:posreach}. In this iteration, the algorithm set $\selector$ to play
$\hactb$ in $s$ with resource level $10$ or more. The final values of $\vec{p}
= \ml{\reachability}{\pos}$ computed by \cref{algo:posreach} are equal to
$\ml{\safety}{}$ for this example (as we can safely reach $t$ from all
reloads). In the iteration, where $\veccomp{\vec{p}}{s} = 2$ for the first
time, the selector is updated to play $\hacta$ in $s$ with resource level
between $2$ and $10$ (excluded). Note that the computed $\selector$ exactly
matches the one mentioned in \cref{ex:selector}.
\end{example}

\begin{theorem}%
\label{thm:posreach-values-main}%
\Cref{algo:posreach} always terminates after a number of steps that is
polynomial with respect to the representation of $\mdp$, and upon termination,
$\vec{p} = \ml{\reachability}{\pos}$.
\end{theorem}

\begin{proof}
The complexity part follows from \cref{lem:posreach-bound} and the fact that
each iteration takes only linear number of steps. The correctness part is an
immediate corollary of \cref{lem:posreach-bound} and the fact that
\cref{algo:posreach} iterates $\mathcal{B}$ on $\vec{y}_\target$ until a
fixed point.
\end{proof}

\begin{theorem}%
\label{thm:posreach-strat-main}%
Upon termination of \cref{algo:posreach}, the computed rule selector
$\selector$ encodes a strategy $\sigma_{\selector}$ such that
$\vecsat{\sigma_\selector}{\vec{v}}{\reachability}{\pos}$ for
$\vec{v}=\ml{\reachability}{\pos}$. As a consequence, a polynomial-size
finite counter strategy for the positive reachability problem can be computed in
time polynomial with respect to representation of $\mdp$.
\end{theorem}

\begin{proof}
The complexity follows from \Cref{thm:posreach-values-main}. Indeed, since the
algorithm has a polynomial complexity, also the size of $ \selector $ is
polynomial. The correctness proof is based on the following invariant of the
main repeat-loop. The vector $\vec{p}$ and the finite counter strategy $\pi =
\sigma_\selector$ have these properties:
\begin{enumerate}[(a)]
\item It holds that $\vec{p} \geq \ml{\safety}{}$.\label{item:pr-safe}%
\item Strategy $\pi$ is safe.\label{item:pr-pi-safe}%
\item For each $s\in\states$ with $d$ such that $\veccomp{\vec{p}}{s} \leq d
\leq \Ca$, there is a finite $\pi$-compatible path
$\loadedpath{\hist}{d} = \loadedpath{s}{d}_1a_1s_2\ldots s_n$ with $s_1=s$ and
$s_n\in\target$ such that $\reslevs{d}{\hist} = r_1, r_2, \ldots, r_n$ never
drops below $\vec{p}$, which is $r_i \geq \veccomp{\vec{p}}{s_i}$ for all $1
\leq i \leq n$.%
\label{item:pr-pi-reach}%
\end{enumerate}

\noindent The theorem then follows from (\ref{item:pr-pi-safe}) and
(\ref{item:pr-pi-reach}) of this invariant and from
\cref{thm:posreach-values-main}.

Clearly, all parts of the invariants hold after the initialization on
\crefrange{aline:pr-sel-init-b}{aline:pr-init-yt-ii}. The first item of the
invariant follows from the definition of $\SVname$ and $\shopeName$. In
particular, if $\vec{p_\old} \geq \ml{\safety}{}$, then
$\SV{\vec{p_\old}}{s}{a} \geq \AV{\ml{\safety}{}}{s}{a} \geq
\veccomp{\ml{\safety}{}}{s}$ for all $s$ and $a$. The part
(\ref{item:pr-pi-safe}) follows from (\ref{item:pr-safe}), as the action
assigned to $\selector$ on \cref{aline:pr-sel-up-e} is safe for $s$ with
$\veccomp{\vec{p}}{s}$ units of resource (again, due to $\veccomp{\vec{p}}{s}
= \SV{\vec{p_\old}}{s}{a} \geq \AV{\ml{\safety}{}}{s}{a}$); hence, only
actions that are safe for the corresponding state and resource level are
assigned to $\selector$. By \cref{lem:safe-act}, $\pi$ is safe.

The proof of (\ref{item:pr-pi-reach}) is more involved. Assume that an
iteration of the main repeat loop was performed. Denote by $ \pi_{\old} $  the
strategy encoded by $\vec{p}$ and $\selector$ from the previous iteration. Let
$s$ be any state such that $\veccomp{\vec{p}}{s} \leq \Ca$. If
$\veccomp{\vec{p}}{s} = \veccomp{\vec{p}_{\old}}{s}$, then
(\ref{item:pr-pi-reach}) follows directly from the induction hypothesis: for
each state $q$, $\selector(q)$ was only redefined for values smaller then
$\veccomp{\vec{p}_\old}{q}$ and thus the history witnessing
(\ref{item:pr-pi-reach}) for $\pi_\old$ is also $\pi$-compatible.

The case where $\veccomp{\vec{p}}{s} < \veccomp{\vec{p}_{\old}}{s}$ is treated
similarly. We denote by $a$ the action $\veccomp{\vec{a}}{s}$ selected on
\cref{aline:pr-act-sel} and assigned to $\selector(s)$ for
$\veccomp{\vec{p}}{s}$ on line \ref{aline:pr-insert}. By definition of
$\SVname$, there must be $t\in\Succ(s,a)$ such that
$\shope{\vec{p_\old}}{s}{a}{t} + \cons(s,a) = \SV{\vec{p_\old}}{s}{a}$ (which
is equal to $\veccomp{\vec{p}}{s}$ before the truncation on
\cref{aline:pr-trunc}). In particular, it holds that $l =
\lreslev{\veccomp{\vec{p}}{s}}{sat} \geq \veccomp{\vec{p_\old}}{t}$ (even
after the truncation). Then, by the induction hypothesis, there is a
$t$-initiated finite path $\histpr$ witnessing (\ref{item:pr-pi-reach}) for
$\pi_\old$. Then, the loaded history
$\loadedpath{\hist}{\veccomp{\vec{p}}{s}}$ with $\hist=sat\histconc\histpr$ is
(i) compatible with $\pi$ and, moreover, (ii) we have that
$\reslevs{\veccomp{\vec{p_\old}}{s}}{\hist}$ never drops below $\vec{p}$.
Indeed, (i) $\selector(s)(\veccomp{\vec{p}}{s}) = a$ (see
\cref{aline:pr-sel-up-e}), and (ii) $\selector$ was only redefined for values
lower than $\vec{p}_\old$ and thus $\pi$ mimics $\pi_{\old}$ from $t$ onward.
For the initial load $\veccomp{\vec{p}}{s} < d' \leq \Ca$ the same arguments
apply. This finishes the proof of the invariant and thus also the proof of
\Cref{thm:posreach-strat-main}.
\end{proof}

\section{Büchi: visiting targets repeatedly}
\label{sec:buchi}

This section solves the almost-sure Büchi problem. As before, for the rest of
this section we fix a CMDP $\mdp=(\states, \actions, \trans, \cons, \reloads,
\Ca)$ and a set $\target \subseteq \states$.

The solution builds on the positive reachability problem similarly to how the
safety problem builds on the nonreloading reachability problem. In particular,
we identify a largest set $\reloads'\subseteq\reloads$ such that from each
$r\in\reloads'$ we can safely reach $\reloads'$ again (in at least one step)
while restricting ourselves only to safe strategies that (i) avoid
$\reloads\setminus\reloads'$ and (ii) guarantee positive reachability of
$\target$ in $\mdp(\reloads')$ from all $r\in\reloads'$.

Intuitively, at each visit of $\reloads'$, such a strategy can attempt to
reach $\target$. With an infinite number of attempts, we reach $\target$
infinitely often with probability 1 (almost surely). Formally, we show that
for a suitable $\reloads'$ we have that $\ml[\mdp]{\buchi[\target]}{\as} =
\ml[\mdp(\reloads')]{\reachability[\target]}{\pos}$ (where $\mdp(\reloads')$
denotes the CMDP defined as $\mdp$ with the exception that $\reloads'$ is the
set of reload states).

Algorithm~\ref{algo:buchi-rel} identifies the suitable set $\reloads'$ using
\cref{algo:posreach} in a similar fashion as \cref{algo:safety} handled safety
using \cref{algo:mininitcons_iterative}. In each iteration, we declare as
non-reload states all states of $\reloads$ from which positive reachability of
$\target$ within $\mdp(\varRel)$ cannot be guaranteed. This is repeated until
we reach a fixed point. The number of iterations is clearly bounded by
$|\reloads|$.

\begin{algorithm}[bt]
\KwIn{CMDP $ \mdp = (\states, \actions, \trans, \cons, \reloads,
\Ca)$ and $\target\subseteq S$}%
\KwOut{The vector $\ml[\mdp]{\buchi[\target]}{\as}$, rule selector
  $\selector$}%
$\varRel \leftarrow \reloads $; $\varToRemove \leftarrow \emptyset$\;%
\Repeat{$\varToRemove = \emptyset$}%
{%
  $\varRel \leftarrow \varRel\smallsetminus \varToRemove $\;
  \tcc{The next 2 lines use \cref{algo:posreach} on $\mdp(\varRel)$ and $\target$.}
  $\vec{b} \leftarrow \ml[\mdp(\varRel)]{\reachability[\target]}{\pos}$\label{aline:b-posreach-sub-val}\;
  $\selector \leftarrow$ witness selector for $\ml[\mdp(\varRel)]{\reachability[\target]}{\pos}$\label{aline:b-posreach-sub-sel}\;
  $\varToRemove \leftarrow \{ r \in \varRel \mid
  \veccomp{\vec{b}}{r} > \Ca \}$\; }%
\Return{$\vec{b},\selector$}
\caption{Computing $\ml{\buchi[\target]}{\as}$ and a corresponding witness rule selector.}\label{algo:buchi-rel}%
\end{algorithm}

\begin{theorem}\label{thm:buchi}
Upon termination of \cref{algo:buchi-rel} we have that for the strategy
$\sigma_\selector$ encoded by $\Phi$ it holds
$\vecsat[\mdp]{\sigma_\selector}{\vec{b}}{\buchi}{\as}$. Moreover, $\vec{b} =
\ml[\mdp]{\buchi}{\as}$. As a consequence, a polynomial-size finite counter
strategy for the almost-sure Büchi problem can be computed in time polynomial
with respect to the representation of $\mdp$.
\end{theorem}

\begin{proof}
The complexity part follows from the fact that the number of iterations of the
repeat-loop is bounded by $|\reloads|$ and from
theorems~\ref{thm:posreach-values-main}~and~\ref{thm:posreach-strat-main}.

For the correctness part, we first prove that
$\vecsat[\mdp(\reloads')]{\sigma_\selector}{\vec{b}}{\buchi}{\as}$. Then we
argue that the same holds also for $\mdp$. Finally, we show that $\vec{b} \leq
\ml[\mdp]{\buchi}{\as}$; the converse inequality follows from
$\vecsat[\mdp]{\sigma_\selector}{\vec{b}}{\buchi}{\as}$.

Strategy $\sigma_\selector$ has finite memory, also
$\vecsat[\mdp(\reloads')]{\sigma_\selector}{\vec{b}}{\reachability}{\pos}$, and
upon termination, $\veccomp{\vec{b}}{r} = 0$ for all $r\in\reloads'$.
Therefore, there is $\theta > 0$ such that upon every visit of some state $r
\in \reloads'$ we have that $\loadedprobm[\mdp(\reloads')]{\sigma_\selector}{r}{0}(\reachability)
\geq \theta$.

As $\mdp(\reloads^\prime)$ is decreasing, every safe infinite run created by
$\sigma_\selector$ in $\mdp(\reloads')$ must visit $\reloads^\prime$
infinitely many times. Hence, with probability 1 we reach $\target$ at least
once. The argument can then be repeated from the first point of visit of
$\target$ to show that with probability 1 we visit $\target$ at least twice,
three times, etc. \emph{ad infinitum.} By the  monotonicity of probability, we
get $\loadedprobm[\mdp(\reloads')]{\sigma_\selector}{s}{d}(\buchi) = 1$ for
all $s$ with $d = \veccomp{\vec{b}}{s} \leq \Ca$ and
$\vecsat[\mdp(\reloads']{\sigma_\selector}{\vec{b}}{\buchi}{\as}$.

Now we show that also $\vecsat[\mdp]{\sigma_\selector}{\vec{b}}{\buchi}{\as}$.
Let $s$ be a state such that $\veccomp{\vec{b}}{s} \leq \Ca$. Clearly, all
$s$-initiated runs loaded by $d \geq \veccomp{\vec{b}}{s}$ that are compatible
with $\sigma_\selector$ in $\mdp(\reloads')$ avoid
$\reloads\smallsetminus\reloads'$. Therefore,
$\lcompatible[\mdp]{\sigma_\Phi}{s}{d} =
\lcompatible[\mdp(\reloads')]{\sigma_\Phi}{s}{d}$ and we also get
$\sdsat[\mdp]{\sigma_\Phi}{s}{d}{\buchi}{\as}$.

It remains to show that $ \vec{b}\leq \ml[\mdp]{\buchi}{\as}$. Assume for the
sake of contradiction that there is a state $s\in\states$ and a strategy
$\sigma$ such that $\sdsat[\mdp]{\sigma}{s}{d}{\buchi}{\as}$ for some
$d<\veccomp{\vec{b}}{s} =
\veccomp{\ml[\mdp(\reloads')]{\reachability}{\pos}}{s}$. Then there must be at
least one $\loadedpath{\hist}{d}$ created by $\sigma$ such that
$\loadedpath{\hist}{d}$ visits $r\in\reloads\setminus\reloads^\prime$ before
reaching $\target$ (otherwise $d\geq \veccomp{\vec{b}}{s}$). Then either (a)
$\veccomp{\ml[\mdp]{\reachability}{\pos}}{r} = \infty$, in which case any $
\sigma$-compatible extension of $\loadedpath{\hist}{d}$ avoids $\target$; or
(b) since $\veccomp{\ml[\mdp(\reloads')]{\reachability}{\pos}}{r} > \Ca$,
there must be an extension of $\hist$ that visits, between the visit of $r$
and $\target$, another $r' \in \reloads\setminus\reloads'$ such
that $r' \neq r$. We can then repeat the argument, eventually reaching
the case (a) or running out of the resource, a contradiction with
$\sdsat[\mdp]{\sigma}{s}{d}{\safety}{}$.%
\end{proof}
\section{Almost-sure reachability}%
\label{sec:as-reach}

In this section, we solve the \emph{almost-sure reachability} problem, which
is computing the vector $\ml{\reachability[\target]}{\as}$ and the
corresponding witness strategy for a given set of target states
$\target\subseteq\states$.

\subsection{Reduction to B\"uchi}

In the absence of the resource constraints, reachability can be viewed as a special case of B\"uchi: we can simply modify the MDP so that playing any action in some target state $ t \in \target $ results into looping in $ t $, thus replacing reachability with an equivalent B\"uchi condition. In consumption MDPs, the transformation is slightly more involved, due to the need to ``survive'' after reaching $ \target $. Hence, for every CMDP $ \mdp $ and a target set $ \target $ we define a new CMDP $ \altmdp{\mdp}{\target} $ so that solving $ \altmdp{\mdp}{\target} $ w.r.t. the B\"uchi objective entails solving $ \mdp $ w.r.t. the reachability objective. Formally, for $ \mdp = (\states, \actions, \trans, \cons, \reloads, \Ca) $ we have $ \altmdp{\mdp}{\target} = (\states', \actions, \trans', \cons', \reloads', \Ca)$, where the differing components are defined as follows:

\begin{itemize}
\item $ \states' = \states \cup \{\sink\} $, where $ \sink \not\in \states $ is a new sink state, i.e. $ \trans'(\sink,a,\sink) = 1 $ for each $ a\in \actions $;
\item we have $ \reloads' = \reloads \cup \{\sink\} $;
\item for each $ t \in \target $ and $ a\in \actions $ we have $ \trans'(t,a,\sink) = 1$ and $\trans'(t,a,s)=0$ for all $s\in\states$;
\item for each $ t \in \target $ and $a \in \actions$ we have $ \cons'(t,a) = \ml[\mdp]{\safety}{}(t) $;
\item we have $ \cons'(\sink,a) = 1 $ for each $ a \in \actions $; and
\item we have $ \trans'(s,a,t) = \trans(s,a,t) $ and $ \cons'(s,a) = \cons(s,a) $ for every $ s \in \states\setminus \target$, every $ a\in \actions $, and every $t\in\states$.
\end{itemize}

\noindent
We can easily prove the following:

\begin{lemma}
\label{lem:as-reduction} For every $ s \in \states$ it holds that $
\ml[\mdp]{\reachability[\target]}{\as}(s) =
\ml[\altmdp{\mdp}{\target}]{\buchi[\{\sink\}]}{\as}(s) $. Moreover, from a
witness strategy for $ \ml[\mdp]{\reachability[\target]}{\as} $ we can
extract, in time polynomial with respect to the representation of $\mdp$, the
witness strategy for $ \ml[\altmdp{\mdp}{\target}]{\buchi[\{\sink\}]}{\as} $
and vice versa.
\end{lemma}
\begin{proof}
Let $ \sigma $ be a witness strategy for $ \ml[\altmdp{\mdp}{\target}]{\buchi[\{\sink\}]}{\as} $ in $ \altmdp{\mdp}{\target} $. Consider a strategy $ \pi $ in $ \mdp $ which, starting in some state $ s $, mimics $ \sigma $ until some $ t \in \target $ is reached (this is possible since $ \altmdp{\mdp}{\target} $ only differs from $ \mdp $ on $ \target\cup\{\sink\} $), and upon reaching $ \target $ switches to mimicking an arbitrary safe strategy. Since $ \sigma $ reaches $ \sink $ and thus also $ \target $ almost-surely, so does $ \pi $. Moreover, since $ \sigma $ is safe, upon reaching a $ t\in \target $ the current resource level is at least $ \ml[{\mdp}{}]{\safety}{}(t) $, since consuming this amount is enforced in the next step. This is sufficient for $ \pi $ to prevent resource exhaustion after switching to a safe strategy.

It follows that $ \ml[\altmdp{\mdp}{\target}]{\buchi[\{\sink\}]}{\as}(s) \geq \ml[\mdp]{\reachability[\target]}{\as}(s) $ for all $ s \in \states. $ The converse inequality can be proved similarly, by defining a straightforward conversion of a witness strategy for $ \ml[\mdp]{\reachability[\target]}{\as} $ into a witness strategy for $ \ml[\altmdp{\mdp}{\target}]{\buchi[\{\sink\}]}{\as}(s)$. The conversion can be clearly performed in polynomial time, with the help of Algorithm~\ref{algo:safety}.
\end{proof}

Hence, we can solve almost-sure reachability for $ \mdp $ by constructing $
\altmdp{\mdp}{\target} $ and solving the latter for almost-sure B\"uchi via
Algorithm~\ref{algo:buchi-rel}. The construction of $ \altmdp{\mdp}{\target} $
can be clearly performed in time polynomial in the representation of $\mdp$
(using Algorithm~\ref{algo:safety} to compute $ \ml[\mdp]{\safety}{} $), hence
also almost-sure reachability can be solved in polynomial time. %
%

\subsection{Almost-sure reachability without model modification}%
In practice, building $\altmdp{\mdp}{\target}$ and translating the synthesized
strategy back to $\mdp$ is inconvenient. Hence, we also present an algorithm
to solve almost-sure reachability directly on $\mdp$. The algorithm consists
of a minor modification of the already presented algorithms.

To argue the correctness of the algorithm, we need a slight generalization of the MDP modification. We call a vector $\vec{v}\in\extNset^\states$ a \emph{sink vector} for $ \mdp $ if and only
if $0 \leq \veccomp{\vec{v}}{s} \leq \veccomp{\ml[\mdp]{\safety}{\as}}{s}$ or
$\veccomp{\vec{v}}{s} = \infty$ for all $s \in \states$. By
$\sinkstates{\vec{v}}$ we denote the set $\{s \in \states \mid
\veccomp{\vec{v}}{s} < \infty\}$ of states with finite value of $\vec{v}$ and
we call each member of this set a \emph{sink entry}. We say that $ \vec{v} $ is a sink vector \emph{for $ T $} if $ \sinkstates{\vec{v}} = T $. Given a CMDP $ \mdp $, target set $ \target $, and sink vector $ \vec{v} $ for $ T $, we define a new CMDP $ \altmdp[\vec{v}]{\mdp}{\target} $ in exactly the same way as $ \altmdp{\mdp}{\target} $, except for the fourth point: for every $ t \in \target$ we put $ \gamma'(t,a) = \vec{v}(t) $ for all $ a\in \actions $. Note that $ \altmdp{\mdp}{\target} =\altmdp[{\ml[\mdp]{\safety}{}}]{\mdp}{\target} $.

\begin{algorithm}[ht]
\KwIn{CMDP $\mdp=(\states, \actions, \trans, \cons, \reloads, \Ca)$, sets of states $ \target,\altarget \subseteq \states,$ a sink vector $\vec{v}\in\extNset^\states$ for $ T $
}
\KwOut{$ \ml[{\altmdp[{\vec{v}}]{\mdp}{\target}}]{\nonreloading[+\altarget\cup\{\sink\}]}{} $ projected to $ S $}
$\vec{x} \leftarrow \vec{\infty}^\states$\;
\Repeat{$\,\vec{x}_{\mathit{old}} = \vec{x}$}{
  $ \vec{x}_\old \leftarrow \vec{x}$\;
  \ForEach{$s \in \states$}{
    %
    \If{$s\in \sinkstates{\vec{v}}$}{
          $\veccomp{\vec{x}}{s} \leftarrow \veccomp{\vec{v}}{s}$\;
        }
    \Else{
      $c \leftarrow%
        \min_{a\in \actions}\AV{\strunc[\altarget]{\vec{x}_\old}}{s}{a}$\;
      \If{$c < \veccomp{\vec{x}}{s}$}{
        $ \veccomp{\vec{x}}{s} \leftarrow c$\;
      }
    }
  }
}
\Return{$\vec{x}$}
\caption{Modified safe sure reachability.}
\label{algo:non-reloading-sinks}
\end{algorithm}

 Given a CMDP $ \mdp $, subsets of states $ \target, \altarget $ of $ \mdp $, and a sink vector $ \vec{v} $ for $ \target $, Algorithm~\ref{algo:non-reloading-sinks} computes the $ S $-components of the vector $ \ml[{\altmdp[{\vec{v}}]{\mdp}{\target}}]{\nonreloading[+\altarget']}{} $, where $ \altarget' = \altarget\cup\{\sink\} $. To see this, denote by $ \vec{x}_i $ the contents of the variable $ \vec{x} $ in the $ i $-th iteration of Algorithm~\ref{algo:non-reloading-sinks} on the input $ \mdp $, $ \target $, $ \altarget $, $ \vec{v} $; and by $ \vec{v}_i $ the contents of variable $ \vec{v} $ in the $ i $-th iteration of an execution of Algorithm~\ref{algo:mininitcons_iterative} on CMDP $ {\altmdp[{\vec{v}}]{\mdp}{\target}} $ with target set $ \altarget \cup \{\sink\} $. A straightforward induction shows that $ \vec{v}_i(s) = \vec{x}_i(s) $ for all $ s\in \states $ and all $ i $.


Then, the vector $\ml[{\altmdp[{\vec{v}}]{\mdp}{\target}}]{\safety{}}{}$ can be computed using a
slight modification of \cref{algo:safety}: on \cref{aline:mcon} use
$\ml[{\altmdp[{\vec{v}}]{\mdp}{\target}}]{\nonreloading[+R'\cup\{\sink\}]}{}$ (projected to $ S $) computed by
\cref{algo:non-reloading-sinks} instead of $\ml{\nonreloading[+\varRel]}{}$ computed by \cref{algo:mininitcons_iterative}. Then, run the modified algorithm on $ \mdp $. The correctness can be argued similarly as for Algorithm~\ref{algo:non-reloading-sinks}: let $ R'_i $ be the contents of $ R' $ in the $ i $-th iteration of Algorithm~\ref{algo:safety} on $ \altmdp[{\vec{v}}]{\mdp}{\target} $; and let $ \tilde{R}'_i $ be the contents of $ R' $ in the $ i $-th iteration of the modified algorithm executed on $ \mdp $. Clearly, $ \sink \in \tilde{R}'_i $ for all $ i $. An induction on $ i $ shows that for all $ i $ we have $ R' = \tilde{R}' \setminus\{\sink\} $, so both algorithms terminate in the same iteration. The correctness follows from the correctness of Algorithm~\ref{algo:safety}.

Now we can proceed to solve almost-sure reachability. Algorithm~\ref{algo:asreach} combines (slightly modified) algorithms~\ref{algo:posreach} and~\ref{algo:buchi-rel} to mimic the solving of B\"uchi objective in $ \altmdp{\mdp}{\target} $ with a single B\"uchi accepting state $ \{\sink\} $.

Lines \ref{aline:posreach-start}-\ref{aline:as-loop-e} correspond to the computation of Algorithm~\ref{algo:posreach}
on $ \altmdp{\mdp}{\target}(\reloads'\cup\{\sink\}) $. To see this, note that $
\altmdp{\mdp}{\target}(\reloads'\cup\{\sink\}) = \altmdp[\vec{v}]{\mdp(\reloads')}{\target}$ for $ \vec{v} =
\ml[\mdp]{\safety}{} $. Hence, line~\ref{aline:prinitb} in Algorithm 3 is replaced by line~\ref{aline:as-mod-safety} in
\Cref{algo:asreach}. Also, on lines~\ref{aline:as:act-sel} and~\ref{aline:as:apply-A} (which correspond to
lines~\ref{aline:pr-act-sel} and~\ref{aline:pr-apply-A} in Algorithm~\ref{algo:posreach}), we use generalized versions
of $\shopeName$ and $\SVname$ that use
alternative \emph{survival values} $\vec{s}$:


\[
\begin{aligned}
\shopeSinks{\vec{x}}{s}{a}{s'}{\vec{s}} &= %
\max_{\substack{%
  t \in \Succ(s,a)\\
  t \neq s'
}} \{\veccomp{\vec{x}}{s'}, \veccomp{\vec{s}}{t}\}\\
\SVSinks{\vec{x}}{s}{a}{\vec{s}} &= \cons(s, a) + \min_{s'\in\Succ(s, a)} \shopeSinks{\vec{x}}{s}{a}{s'}{\vec{s}}
\end{aligned}
\]

\noindent
Note that $\shopeName =
\shopeName[\ml{\safety}{}]$ and $\SVname = \SVname[\ml{\safety}{}]$. These generalized operators are used because
\Cref{algo:asreach} works on $ \mdp $,
but lines \ref{aline:as:act-sel} and~\ref{aline:as:apply-A} should emulate the computation of lines
\ref{aline:pr-act-sel} and~\ref{aline:pr-apply-A} on $ \altmdp{\mdp}{\target}(\reloads'\cup\{\sink\}) $, so the vector
$ \ml[\mdp]{\safety}{} $ in the definition of the hope value $ \shopeName $ has to be substituted for
$ \ml[{\altmdp{\mdp}{\target}(\reloads'\cup\{\sink\})}]{\safety}{}  =
\ml[{\altmdp[\vec{v}]{\mdp(\reloads')}{\target}}]{\safety}{}$ where $ \vec{v} =  \ml[\mdp]{\safety}{} $.

\begin{algorithm}[t]
  \KwIn{CMDP $\mdp=(\states, \actions, \trans, \cons, \reloads, \Ca)$ and
  $\target\subseteq \states$}%
  \KwOut{The vector $\ml[\mdp]{\reachability[\target]}{\as}$, rule selector
  $\selector$}%
  $\vec{o} \leftarrow \ml[\mdp]{\safety}{}$\tcc*[r]{\scriptsize\cref{algo:safety}}%
  \label{aline:as:orig-safety}%
  $\vec{v} \leftarrow \{\infty\}^S  $\;\label{aline:as:init-vT-b}%
  \lForEach{$t\in \target$}%
  {$\veccomp{\vec{v}}{t} \leftarrow \veccomp{\vec{o}}{t}$}%
  $\reloads' \leftarrow \reloads$; $\varToRemove \leftarrow \emptyset$\;
  \Repeat{$\varToRemove = \emptyset$}%
  {%
    $\varRel \leftarrow \varRel\smallsetminus \varToRemove $\;
    $\selector \leftarrow \text{ an empty selector}$\label{aline:posreach-start}\;
    \ForEach{$s\in \states$\label{aline:as:sel-init-b}}%
      {
        $\selector(s)(\veccomp{\vec{o}}{s}) \leftarrow \text{arbitrary \minsafe{} action of }s$
      }\label{aline:as-sel-init-e}%
      $\vec{p} \leftarrow \vec{v}$\;%

      $\vec{s} \leftarrow \ml[{\altmdp[{\vec{v}}]{\mdp(R')}{\target}}]{\safety}{}$\tcc*[r]{\scriptsize modified
      alg.~\ref{algo:safety}} \label{aline:as-mod-safety}
      \Repeat{$\,\vec{p}_{\mathit{old}} = \vec{p}$\label{aline:as:loop-e}}{
        \label{aline:as:loop-b}%
        $\vec{p}_\old \leftarrow \vec{p}$\;%
        \ForEach{$s \in \states\smallsetminus \target$}{%
          $\veccomp{\vec{a}}{s} \leftarrow%
            \arg\min_{a\in \actions} \SVSinks{\vec{p}_\old}{s}{a}{\vec{s}}$\;\label{aline:as:act-sel}%
          $\veccomp{\vec{p}}{s} \leftarrow%
            \min_{a\in \actions} \SVSinks{\vec{p}_\old}{s}{a}{\vec{s}}$\;\label{aline:as:apply-A}%
        }%
        $\vec{p} \leftarrow \trunc{\vec{p}}_{\reloads'}^{\Ca}$\;\label{aline:as:trunc}%
        \ForEach{$ s\in \states\smallsetminus \target $\label{aline:as:sel-up-b}}{%
          \If{$ \veccomp{\vec{p}}{s} < \veccomp{\vec{p}_{\old}}{s}$}{%
           $\selector(s)(\veccomp{\vec{p}}{s}) \leftarrow%
           \veccomp{\vec{a}}{s} $\;\label{aline:as-insert}\label{aline:as-sel-up-e}%
          }%
        }%
      }\label{aline:as-loop-e}%
    $\varToRemove \leftarrow \{ r \in \varRel \mid
    \veccomp{\vec{p}}{r} > \Ca \}$\; }%
  \Return{$\vec{p},\selector$}
  \caption{Computing $\ml{\reachability[\target]}{\as}$ and a corresponding witness rule selector.} \label{algo:asreach}
\end{algorithm}

Hence, the respective lines indeed emulate the computation of~\Cref{algo:posreach} on
$\altmdp{\mdp}{\target}(\reloads'\cup\{\sink\})$. It remains to show that the whole repeat loop emulates the
computation of~\Cref{algo:buchi-rel} on $ \altmdp{\mdp}{\target} $. But this follows immediately from the fact that in
the latter computation, $ \sink $ always stays in $ R' $.

\section{Improving Expected Reachability Time}
\label{sec:heuristics}

The number of steps that a strategy needs on average to reach the target set
$\target$ (\emph{expected reachability time (ERT)}) is a property of practical
importance. %
For example, we expect that a patrolling unmanned vehicle visits all the
checkpoints in a reasonable amount of time. %
The presented approach is purely qualitative (ensures reachability with
probability 1) and thus does not consider the number of steps at all. In this
section, we introduce two heuristics that can improve ERT: the
\emph{goal-leaning} heuristic and the \emph{threshold} heuristic. These slight
modifications of the algorithms produce strategies that can often hit
$\target$ sooner than the strategies produced by the unmodified algorithms.

\subsection{Expected reachability time}

To formally define ERT, we introduce a new objective: $\minreach[\target]{i}$
(\emph{reachability first in $i$ steps}) as %
$\minreach[\target]{i} = \{%
\loadedpath{\run}{d} \in \reachability[\target]^i \mid%
\loadedpath{\run}{d} \notin \reachability[\target]^j%
\text{ for all } 0 \leq j < i\}$ %
(the set of all safe loaded runs $\loadedpath{\run}{d}$ such that the minimum
$j$ such that $\run_j\in\target$ is equal to $i$). Finally, the \emph{expected
reachability time} for a a strategy $\sigma$, an initial state $s\in\states$,
an initial load $d\leq \Ca$, and a target set $\target\subseteq\states$ is
defined as follows.
\[
\ERT[\mdp]{\sigma}{s}{d}{\target} = \sum_{i \in \Nset} i \cdot \loadedprobm[\mdp]{\sigma}{s}{d}(\minreach[\target]{i})
\]

The running example for this section is the CMDP in \cref{fig:heurex-i} with
capacity $\geq 3$ and the almost-sure satisfaction of the reachability
objective for $\target=\{t\}$. Consider the following two memoryless
strategies that differ only in the action played in $s$: $\sigma_{\hacta}$
always plays $\hacta$ in $s$ and $\sigma_{\hactb}$ always plays $\hactb$.
Loaded with $2$ units of resource in $s$, we have
$\ERT{\sigma_{\hacta}}{s}{2}{\{t\}}=2$ and
$\ERT{\sigma_{\hactb}}{s}{2}{\{t\}}=20$. Indeed, $\sigma_{\hacta}$ surely
reaches $t$ in $2$ steps while $\sigma_{\hactb}$ needs $10$ trials on average
before reaching $t$ via $v$, each trial needing $2$ steps before coming back
to $s$.

\begin{figure}[h]
\centering
\begin{tikzpicture}[automaton]
\tikzstyle{a} = [magenta]
\tikzstyle{b} = [cyan]
\tikzstyle{ab} = [black, thick]
\node[state, initial, initial angle=135] (start)    at (0,0)  {$s$};
\node[state,reload,target] (t)   at (4.5,0) {$t$};
\node[state] (u)      at (2.25,1)  {$u$};
\node[state, reload] (rel) at (-2.25,0) {$ r $};
\node[state] (v) at (2.25,-1)  {$v$};

\path[->,auto, swap]
(start) edge[bend left=25, a]
  node[swap] {$\acta$}
  node[cons] {1}
(u)
(v) edge[bend right=25, ab]
  node[swap] {$\acta,\actb$}
  node[cons] {0}
(t)
(u) edge[bend left=25, ab]
  node[swap] {$\acta,\actb$}
  node[cons] {1}
(t)
(start) edge[out = -30, in = 180, looseness = 1.1, b]
  node[above, pos=0.18] {$\actb$}
  node[prob, overlay] {$ \frac{1}{10} $}
(v)
(start) edge[out = -30, in = 290, looseness = 1.1, b]
  node[prob,swap,pos=.7] {$ \frac{9}{10} $}
  node[cons, swap, pos=0.3,] {2}
(rel)
(rel) edge[ab]
  node[swap] {$\acta,\actb$}
  node[cons] {1}
(start)
(t) edge[loop right, ab]
  node[above, outer sep=3pt] {$\acta,\actb$}
  node[cons, below, outer sep=3pt] {1}
(t)
;

\begin{scope}[every node/.append style={valbox, above}]
\node[as reach] at (start) {$2$};
\node[as reach] at (t) {$0$};
\node[as reach] at (u) {$1$};
\node[as reach] at (v) {$0$};
\node[as reach] at (rel) {$0$};
\end{scope}

\end{tikzpicture}
\caption{Example CMDP for the goal-leaning heuristic. The green boxes above states indicate $\ml{\reachability[\{t\}]}{\as}$.}
\label{fig:heurex-i}
\end{figure}
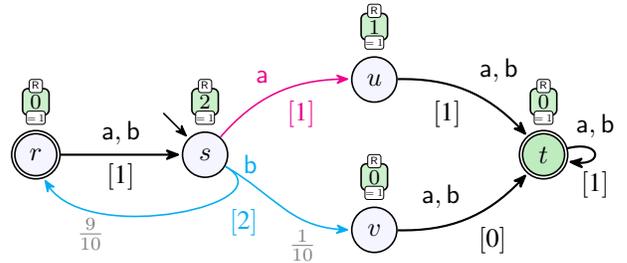

\subsection{Goal-leaning heuristic}

Actions to play in certain states with a particular amount of resource are
selected on \cref{aline:pr-act-sel} of \cref{algo:posreach} (and on
\cref{aline:as:act-sel} of \cref{algo:asreach}) based on the actions'
\emph{save values $\SVname$}. This value in $s$ based on
$\ml{\reachability}{\as}$ is equal to 2 for both actions $\hacta$ and $\hactb$.
Thus, \cref{algo:posreach} (and also \cref{algo:asreach}) returns
$\sigma_{\hacta}$ or $\sigma_{\hactb}$ randomly based on the resolution of the
$\arg\min$ operator on \cref{aline:pr-act-sel} (\cref{aline:as:act-sel} in
\cref{algo:asreach}) for $s$. The \emph{goal-leaning heuristic} fixes the
resolution of the $\arg \min$ operator to always pick $\hacta$ in this example.

The ordinary $\arg\min$ operator selects randomly an action from the pool of
actions with the minimal value $v_{\min}$ of the function $\SVname$ for $s$
(and the current values of $\vec{p_\old}$). Loosely speaking, the goal-leaning
$\arg\min$ operator chooses, instead, the action whose chance to reach the
desired successor used to obtain $v_{\min}$ is maximal among actions in this
pool.

The value $\SVname$ is computed using successors' hope values ($\shopeName$),
see \cref{sec:posreach}. The goal-leaning $\arg\min$ operator records, when
computing the $\shopeName$ values, also the transition probabilities of the
desired successors. Let $s$ be a state, let $a$ be an action, and let
$t\in\Succ(s,a)$ be the successor of $a$ in $s$ that minimizes
$\shope{\vec{p_\old}}{s}{a}{t}$ \emph{and} maximizes $\trans(s,a,t)$ (in this
order). We denote by $p_{s,a}$ the value $\trans(s,a,t)$. The goal-leaning
$\arg\min$ operator chooses the action $a$ in $s$ that minimizes
$\SV{\vec{p_\old}}{s}{a}$ \emph{and} maximizes $p_{s,a}$.

In the example from \cref{fig:heurex-i} we have that $p_{s,\hacta}=1$ and
$p_{s,\hactb}=\frac{1}{10}$ (as $v$ is the desired successor) in the second
iteration of the repeat-loop on \crefrange{aline:pr-loop-b}{aline:pr-loop-e}.
In the last iteration, $p_{s,\hacta}$ remains $1$ and $p_{s,\hactb}$ changes to
$\frac{9}{10}$ as the desired successor changes to $r$. In both cases, $\hacta$
is chosen by the goal-leaning $\arg\min$ operator as $p_{s,\hacta} >
p_{s,\hactb}$.

\textbf{Correctness.} We have only changed the behavior of the $\arg\min$
operator when multiple candidates could be used. The correctness of our
algorithms does not depend on this choice and thus the proofs apply also to
the variant with the goal-leaning operator.

\medskip

While the goal-leaning heuristic is simple, it has a great effect in practical
benchmarks; see \cref{sec:experiments}. However, there are scenarios where it
still fails. Consider now the CMDP in \cref{fig:heurex-ii} with capacity at
least $3$. Note that now $\cons(s,\hactb)$ equals to $1$ instead of $2$. In
this case, even the goal-leaning heuristic prefers $\hactb$ to $\hacta$ in $s$
whenever the current resource level is at least $1$. The reason for this
choice is that $\SV{\vec{p_\old}}{s}{\hactb} = 1 < 2 =
\SV{\vec{p_\old}}{s}{\hacta}$ from the second iteration of the repeat-loop
onward.

\begin{figure}[hb]
\centering
\begin{tikzpicture}[automaton]
\tikzstyle{a} = [magenta]
\tikzstyle{b} = [cyan]
\tikzstyle{ab} = [black, thick]
\node[state, initial, initial angle=135] (start)    at (0,0)  {$s$};
\node[state,reload,target] (t)   at (4.5,0) {$t$};
\node[state] (u)      at (2.25,1)  {$u$};
\node[state, reload] (rel) at (-2.25,0) {$ r $};
\node[state] (v) at (2.25,-1)  {$v$};

\path[->,auto, swap]
(start) edge[bend left=25, a]
  node[swap] {$\acta$}
  node[cons] {1}
(u)
(v) edge[bend right=25, ab]
  node[swap] {$\acta,\actb$}
  node[cons] {0}
(t)
(u) edge[bend left=25, ab]
  node[swap] {$\acta,\actb$}
  node[cons] {1}
(t)
(start) edge[out = -30, in = 180, looseness = 1.1, b]
  node[above, pos=0.18] {$\actb$}
  node[prob, overlay] {$ \frac{1}{10} $}
(v)
(start) edge[out = -30, in = 290, looseness = 1.1, b]
  node[prob,swap,pos=.7] {$ \frac{9}{10} $}
  node[cons, swap, pos=0.3,] {1}
(rel)
(rel) edge[ab]
  node[swap] {$\acta,\actb$}
  node[cons] {1}
(start)
(t) edge[loop right, ab]
  node[above, outer sep=3pt] {$\acta,\actb$}
  node[cons, below, outer sep=3pt] {1}
(t)
;

\begin{scope}[every node/.append style={valbox, above}]
\node[as reach] at (start) {$1$};
\node[as reach] at (t) {$0$};
\node[as reach] at (u) {$1$};
\node[as reach] at (v) {$0$};
\node[as reach] at (rel) {$0$};
\end{scope}

\end{tikzpicture}
\caption{Example CMDP for the threshold heuristic. In comparison to
\cref{fig:heurex-i}, the consumption of $\hactb$ is $1$ instead of $2$.}
\label{fig:heurex-ii}
\end{figure}
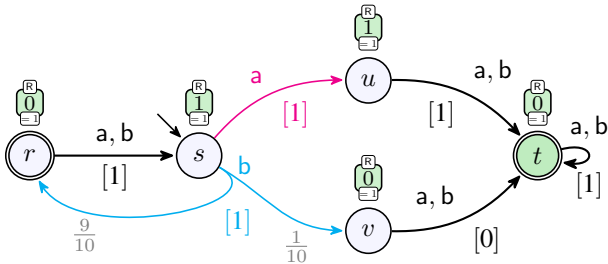

Note also that the strategy $\sigma_{\hacta}$ that always plays $\hacta$ in
$s$ is not a witness strategy for $\ml{\reachability}{\as}$ as
$\sigma_{\hacta}$ needs at least $2$ units of resource in $s$. The desired
strategy $\pi$ should behave in $s$ as follows: play $\hacta$ if the current
resource is at least $2$ and otherwise play $\hactb$. We have that
$\ERT{\pi}{s}{2}{\{t\}} = 2$ and $\ERT{\pi}{s}{1}{\{t\}} = 3.8$. In the next
section, we extend the goal-leaning heuristic to produce $\pi$ for the
(updated) running example.

\subsection{Threshold heuristic}

The threshold heuristic is parametrized by a probability threshold $0 \leq
\theta \leq 1$. Intuitively, when we compute the value of $\SVname$ for
$\hactb$ in $s$, we ignore the hope values of successors $t\in
\Succ(s,\hactb)$ such that $\trans(s,\hactb,t) < \theta$. With $\theta = 0.2$,
$v$ in our example is no longer considered as a valid outcome for $\hactb$ in
$s$ in the second iteration and $\SV{\vec{p_\old}}{s}{\hactb}=\infty$.
Therefore $\hacta$ is picked with $\SV{\vec{p_\old}}{s}{\hacta}=2$. It happens
only in the fourth iteration that action $\hactb$ is considered from $s$. In
this iteration, $\veccomp{\vec{p_\old}}{r}$ is $0$ (it is a reload state) and
with $\trans(s,\hactb,r)=0.9$, $r$ passes the threshold and we finally have
that $\SV{\vec{p_\old}}{s}{\hactb}=1$. The resulting finite counter strategy is
exactly the desired strategy $\pi$ from above.

Formally, we parametrize the function $\SVname$ by $\theta$ as follows where
we assume $\min$ of the empty set is equal to $\infty$ (changes to definition
of $\SVname$ are highlighted in \textcolor{red!70!black}{red}).
\[
  \SVThres{\vec{x}}{s}{a}{\textcolor{red!70!black}{\theta}} = \cons(s,a) +
    \min_{\substack{
      s'\in\Succ(s,a)\\
      \textcolor{red!70!black}{\trans(s,a,s') \geq \theta}
    }} \shope{\vec{x}}{s}{a}{s'}
\]

The new function $\SVname_{\!\theta}$ is a generalization of $\SVname =
\SVname_{\!0}$. To implement this heuristic, we need, in addition to the
goal-leaning $\arg\min$ operator, to use $\SVname_{\!\theta}$ instead of
$\SVname$ in \cref{algo:posreach,algo:asreach}.

There is, however, still one caveat introduced by the threshold. By ignoring
some outcomes, the threshold heuristic might compute only over-approximations
of $\ml{\reachability}{\pos}$. As a consequence, the strategy $\sigma$
computed by the heuristic might be incomplete; it might be undefined for a
resource level from which the objective is still satisfiable.

In order to make $\sigma$ complete and to compute $\ml{\reachability}{\pos}$
precisely, we continue with the iterations, but now using $\SVname_{\!0}$
instead of $\SVname_{\!\theta}$. To be more precise, we include
\crefrange{aline:pr-loop-b}{aline:pr-loop-e} in \cref{algo:posreach} twice
(and analogously for \cref{algo:asreach}), once with $\SVname_{\!\theta}$ and
once with $\SVname_{\!0}$ (in this order).

This extra fixed-point iteration can complete $\sigma$ and improve $\vec{p}$ to
match $\ml{\reachability}{\pos}$ using the rare outcomes ignored by the
threshold. As a result, $\sigma$ behaves according to the threshold heuristic
for sufficiently high resource levels and, at the same time, it achieves the
objective from every state-level pair where this is possible.

\textbf{Correctness.} The function $\SVname_{\!\theta}$ clearly
over-approximates $\SVname$ as we restrict the domain of the $\min$ operator
only. The invariant of the repeat-loop from the proof of
\cref{thm:posreach-strat-main} still holds even when using
$\SVname_{\!\theta}$ instead of $\SVname$ (it also obviously holds in the
second loop with $\SVname_{\!0}$). The extra repeat-loop with $\SVname_{\!0}$
converges to the correct fixed point due to the monotonicity of $\vec{p}$ over
iterations. Thus, \cref{thm:posreach-values-main,thm:posreach-strat-main} hold
even when using the threshold heuristics.

\subsection{Limitations}

The suggested heuristics naturally do not always produce strategies with the
least ERT possible for given CMDP, state, and initial load. Consider the CMDP
in \cref{fig:heurex-iii} with capacity at least $2$. Both heuristics prefer
(regardless $\theta$) $\hactb$ in $s$ since $\trans(s,\hactb,v) >
\trans(s,\hacta,v) = \trans(s,\hacta,u)$. Such strategy yields ERT from $s$
equal to $2\frac{2}{3}$, while the strategy that plays $\hacta$ in $s$ comes
with ERT equal to $2$.

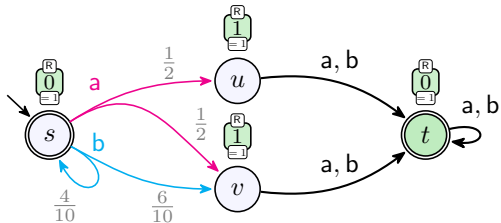
\begin{figure}[h!]
\centering
\begin{tikzpicture}[automaton, yscale=.9]
\tikzstyle{a} = [magenta]
\tikzstyle{b} = [cyan]
\tikzstyle{ab} = [black, thick]
\node[state, initial, initial angle=135, reload] (start) at (0,0)  {$s$};
\node[state,reload,target] (t)   at (5,0) {$t$};
\node[state] (u)      at (2.5,.8)  {$u$};
\node[state] (v) at (2.5,-.8)  {$v$};

\path[->,auto, swap]
(start) edge[out=35, in=180, a]
  node[swap, pos=.3] {$\acta$}
  node[prob, swap] {$\frac{1}{2}$}
(u)
(start) edge[out=35, a, looseness=1.3]
  node[prob, swap] {$\frac{1}{2}$}
(v)

(v) edge[bend right=25, ab]
  node[above] {$\acta,\actb$}
(t)
(u) edge[bend left=25, ab]
  node[above] {$\acta,\actb$}
(t)
(start) edge[out = -30, in = 180, looseness = 1.1, b]
  node[above, pos=0.18] {$\actb$}
  node[prob, overlay] {$ \frac{6}{10} $}
(v)
(start) edge[out = -30, in = 290, looseness = 10, b]
  node[prob,swap,pos=.7] {$ \frac{4}{10} $}
(start)
(t) edge[loop right, ab]
  node[above, outer sep=3pt] {$\acta,\actb$}
(t)
;

\begin{scope}[every node/.append style={valbox, above}]
\node[as reach] at (start) {$0$};
\node[as reach] at (t) {$0$};
\node[as reach] at (u) {$1$};
\node[as reach] at (v) {$1$};
\end{scope}
\end{tikzpicture}
\caption{CMDP illustrating limitations of the goal-leaning heuristics. All
actions consume $1$ unit of resource.}%
\label{fig:heurex-iii}
\end{figure}

This non-optimality must be expected as the presented algorithm is purely
qualitative and does not convey any quantitative analysis that is required to
compute the precise ERT of strategies. However, there is no known polynomial
(with respect to the CMDP representation) algorithm for quantitative analysis
of CMDPs that we could use here instead of our approach.

While other, perhaps more involved heuristics might be invented to solve some
particular cases, qualitative algorithms which do not track precise values of
ERT, naturally cannot guarantee optimality with respect to ERT. The presented
heuristics are designed to be simple (both in principle and computation
overhead) and to work well on systems with rare undesired events.

The threshold heuristic relies on a well-chosen threshold $\theta$. This
threshold needs to be supplied by the user. Moreover, different thresholds
work well for different models. Typically, $\theta$ should be chosen to be
higher than the probability of the most common rare events in the model, to
work well. As the presented algorithms rely on the fact that the whole model
is known, a suitable threshold might be automatically inferred from the model.

Despite these limitations, we show the utility of the presented heuristics on
a case study in the next section.
\section{Implementation and evaluation}
\label{sec:experiments}

We have implemented \crefrange{algo:mininitcons_iterative}{algo:asreach},
including the proposed heuristics in a tool called \FiMDP{} (Fuel in MDP). The
rest of this section presents two numerical examples that demonstrate utility
of \FiMDP{} on realistic environments. In particular, we first compare the
speed of strategy synthesis via CMDPs performed by \FiMDP{} to the speed of
strategy synthesis via regular MDPs with energy constraints encoded in states,
performed by \Storm. The second example shows the impact of heuristics from
\cref{sec:heuristics} on expected reachability time. Jupyter notebooks at
\url{https://github.com/FiMDP/FiMDP-evaluation/tree/tac} contain (not only)
scripts and instructions needed to reproduce the presented results.

\subsection{Tools, examples, and evaluation setting}

\FiMDP{} is an open-source library for CMDPs. It is written in Python and is
well integrated with interactive Jupyter notebooks \cite{Kluyver:2016aa} for
visualization of CMDPs and algorithms. With \Storm~\cite{storm} and \Stormpy{}
installed, \FiMDP{} can read models in \Prism{} \cite{prism} or \Jani{}
\cite{jani} languages.

\Storm{} is an open-source, state-of-the-art probabilistic model checker
designed to be efficient in terms of time and memory. \Storm{} is written in
C++ and \Stormpy{} is its Python interface. The examples are based on models
generated by \FiMDPEnv{} --- a library of simulation environments for real-world
resource-constrained problems that can be solved via CMDPs.
\Cref{tab:versions} lists the homepages of these tools and versions used to
create the presented results.

We demonstrate the utility of CMDPs and \FiMDP{} on high-level planning tasks
for unmanned underwater vehicles (UUVs) operating in ocean with stochastic
currents. \FiMDPEnv{} models this scenario based on~\cite{al2012extending}. The
model discretizes the area of interest into a 2D grid-world. A grid-world of
size $n$ consists of $n\times n$ cells, see \cref{fig:gridworld} (left). Each
cell in the grid-world forms one state in the corresponding CMDP, some of them
are reload states, and some of them form the set of targets $\target$. The set
of actions consists of two classes of actions: \emph{weak actions} consume
less energy but have stochastic outcomes whereas \emph{strong actions} have
deterministic outcomes with the downside of significantly higher resource
consumption. For each class, the environment offers up to 8 directions (east,
north-east, north, north-west, west, south-west, south, and south-east), see \cref{fig:gridworld} (right).

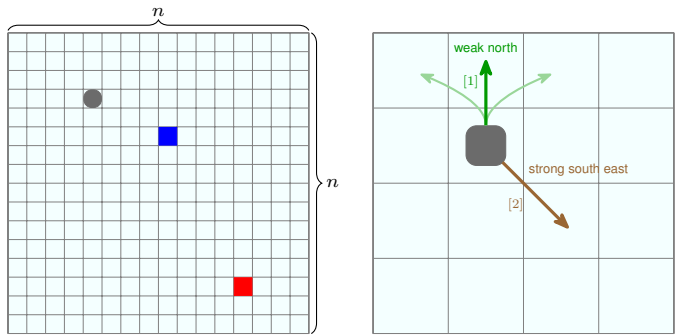
\begin{figure}[hb]
\def\gridsize{4}
\tikzstyle{strong} = [blue]
\tikzstyle{strong} = []

\begin{tikzpicture}
\tikzstyle{brace} = [decorate,
        decoration={brace, amplitude=4pt, raise=0.2ex}]
\draw[gridworld, fill=background!30, step=.25] (0,0) grid (\gridsize,\gridsize) rectangle (0,0);

\tikzstyle{target} = [fill, draw, red]
\tikzstyle{reload} = [fill, draw, blue]
\tikzstyle{state} = [minimum size=.25, xshift=.1255cm, yshift=.1245cm]

\node[agent, minimum size=4pt, state, rounded corners=3pt] at (1, 3) (agent) {};

\node[state, reload] at (2,2.5) {};
\node[state, target] at (3, .5) {};

\draw[brace, overlay] (0, \gridsize) -- node[above=3pt, font=\scriptsize] {$n$} (\gridsize,\gridsize);
\draw[brace, overlay] (\gridsize,\gridsize) -- node[right=3pt, font=\scriptsize] {$n$} (\gridsize, 0);

\end{tikzpicture}
\hfill
\begin{tikzpicture}
\tikzstyle{strong} = [brown!80!black]

\draw[gridworld, fill=background!30, step=1] (0,0) grid (\gridsize,\gridsize) rectangle (0,0);

\node[agent] at (1.5, 2.5) (agent) {};

\path[->, action edge]
(agent)
  edge[out=90, in=200, side effect] +(1,1)
  edge[out=90, in=-20, side effect] +(-1,1)
  edge[main]
  node[action name, left, pos=.8] {$[1]$}
  +(0,1)
  edge[main, strong, shorten <=-2pt]
    node [action name, above right, outer sep=0pt, pos=.3] {\textsf{strong south east}} node[action name, below left, pos=.45] {$[2]$} +(1,-1)
;
\node[main, action name] at (1.5, 3.8) {\textsf{weak north}};
\end{tikzpicture}
\caption{Grid-world of size $n$ with an \textcolor{black!70}{agent} (UUV), a \textcolor{blue}{reload state}, and a \textcolor{red}{target} (left), and illustration of \textcolor{green!60!black}{weak}, and \textcolor{brown!80!black}{strong} actions (right)}
\label{fig:gridworld}
\end{figure}

All experiments were performed on a PC with Intel Core i7-8700 3.20GHz 12 core
processor and with 16 GB RAM running Ubuntu 18.04 LTS. \Cref{tab:versions} lists tools and versions used to obtain the presented results.

\begin{table}
\setlength{\tabcolsep}{2pt}
\centering
\caption{Versions of tools used in numerical examples.}%
\label{tab:versions}
\begin{tabular}{lcl}
\toprule
tool & version & homepage\\
\midrule
\FiMDP    & 2.0.0& \url{https://github.com/FiMDP/FiMDP}\\
\FiMDPEnv & 1.0.4& \url{https://github.com/FiMDP/FiMDPEnv}\\
\Storm    & 1.6.2& \url{https://stormchecker.org}\\
\Stormpy  & 1.6.2& \url{https://moves-rwth.github.io/stormpy}\\
\bottomrule
\end{tabular}
\end{table}

\subsection{Strategy synthesis for CMDPs in \FiMDP{} and \Storm}

We use the UUV environment from \FiMDPEnv{} to generate 15 strategy synthesis
tasks with a Büchi objective. The complexity of a task is determined by grid
size (the number of cells on each side) and capacity. We use grid sizes $10$,
$20$, and $50$. For each grid size $n$, we create five tasks with capacities
equal to $1$, $2$, $3$, $5$, and $10$ times $n$.%
We solve each task modeled as a CMDP using \FiMDP{} and modeled as a regular
MDP with resource constraints encoded in states and actions using \Storm{}. We
express the qualitative Büchi property in PCTL~\cite{BaiKat08} for \Storm.
Figure \ref{fig:gw-comptime} presents the running times (averaged over 10
independent runs) needed for each task by \FiMDP{}
(\textcolor{fimdpcolor}{$\bullet$}) and by \Storm{}
(\textcolor{stormcolor}{\boldmath{$\times$}}). For each grid size we have one
plot and a dot $(x,y)$ indicates that the corresponding tool needed $y$
seconds on average for the task with capacity $x$.

\begin{figure}[ht]
	\centering \scalebox{.7}{
%
\definecolor{mycolor1}{rgb}{0.00000,0.44700,0.74100}%
\definecolor{mycolor2}{rgb}{0.74100,0.00000,0.44700}%
\begin{tikzpicture}
\def\plotsize{1.1in}

\begin{axis}[%
width=\plotsize,
height=\plotsize,
at={(0,0)},
scale only axis,
xlabel style={font=\color{white!15!black}},
xlabel={capacity},
ylabel style={font=\color{white!15!black}},
yticklabel style={/pgf/number format/fixed, /pgf/number format/precision=3},
scaled y ticks=false,
ylabel={comp. time (sec)},
axis background/.style={fill=white},
title style={font=\bfseries},
title={(a) Grid size 10},
legend style={legend cell align=left, align=left, draw=white!15!black, font=\scriptsize},
legend columns=1, legend pos=north west,
legend entries={FiMDP, Storm}
]
\addplot[only marks, mark=*, mark options={}, mark size=2pt, color=mycolor1, fill=mycolor1] table[col sep=comma,x index=2, y index=5] {data/gs10_timedata.csv};
\addplot[only marks, mark=x, mark options={thick}, mark size=2pt, color=mycolor2, fill=mycolor2] table[col sep=comma,x index=2, y index=6] {data/gs10_timedata.csv};

\end{axis}

\begin{axis}[%
width=\plotsize,
height=\plotsize,
at={(4cm,0)},
scale only axis,
xlabel style={font=\color{white!15!black}},
xlabel={capacity},
axis background/.style={fill=white},
title style={font=\bfseries},
title={(b) Grid size 20},
legend style={legend cell align=left, align=left, draw=white!15!black, font=\scriptsize},
legend columns=1, legend pos=north west,
legend entries={FiMDP, Storm}
]
\addplot[only marks, mark=*, mark options={}, mark size=2pt, color=mycolor1, fill=mycolor1] table[col sep=comma,x index=2, y index=5] {data/gs20_timedata.csv};
\addplot[only marks, mark=x, mark options={thick}, mark size=2pt, color=mycolor2, fill=mycolor2] table[col sep=comma,x index=2, y index=6] {data/gs20_timedata.csv};

\end{axis}

\begin{axis}[%
width=\plotsize,
height=\plotsize,
at={(8cm,0)},
scale only axis,
xlabel style={font=\color{white!15!black}},
xlabel={capacity},
axis background/.style={fill=white},
title style={font=\bfseries},
title={(c) Grid size 50},
legend style={legend cell align=left, align=left, draw=white!15!black, font=\scriptsize},
legend columns=1, legend pos=north west,
legend entries={FiMDP, Storm}
]
\addplot[only marks, mark=*, mark options={}, mark size=2pt, color=mycolor1, fill=mycolor1] table[col sep=comma,x index=2, y index=5] {data/gs50_timedata.csv};
\addplot[only marks, mark=x, mark options={thick}, mark size=2pt, color=mycolor2, fill=mycolor2] table[col sep=comma,x index=2, y index=6] {data/gs50_timedata.csv};

\end{axis}

\pgfplotsset{scaled y ticks=false}
\end{tikzpicture}%
}%
	\caption{Mean computation times for solving the CMDP model of the UUV environment with capacities proportional to the grid size in each task. Each subplot in the figure corresponds to a different size of the grid-world.}
	\label{fig:gw-comptime}
\end{figure}
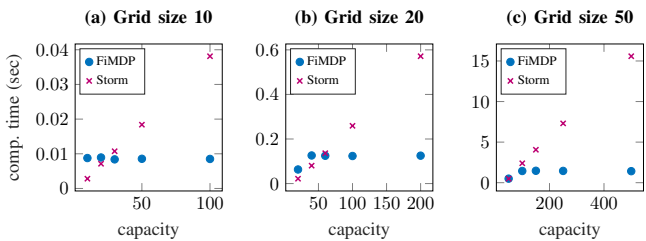

We can observe that \FiMDP{} outperforms \Storm{} in terms of computation time
in all test cases with the exception of small problems. For the small tasks,
\Storm{} benefits from its efficient implementation in C++. The advantage of
\FiMDP{} lies in the fact that the state space of CMDPs (and also the time
needed for their analysis) does not grow with rising capacity.

\subsection{Comparing heuristics for improving ERT}

This section investigates the novel heuristics from a practical, optimal
decision-making perspective. The test scenario is based in the UUV environment
with grid size $20$, a single reload state and one target state. The objective
of agents is to reach the target almost surely. We consider four strategies
generated for almost-sure reachability by the presented algorithms: the
standard strategy (using randomized $\arg\min$ operator), and strategies
generated using the goal-leaning $\arg\min$ operator and thresholds $\theta$
equal to $0$, $0.3$, and $0.5$. For each strategy, we run $10000$ independent
runs and measure the number of steps needed to reach the target. By averaging
the collected data, we approximate the expected reachability time (ERT) for
the strategies.

\begin{table}[hb]
\caption{Approximated ERT for strategies with standard and goal-leaning $\arg\min$ operators}
\label{tab:heuristics}
\centering
\begin{tabular}{lcr}
\toprule
\textbf{operator} & $\mathbf{\theta}$ & \textbf{ERT} \\
\midrule
standard & --- & $200$+\hphantom{$.0$}\\
goal-leaning & $0\hphantom{.2}$ & $51.27$\\
goal-leaning & $0.3$            & $19.53$ \\
goal-leaning & $0.5$            & $15.00$ \\
\bottomrule
\end{tabular}
\end{table}

\Cref{tab:heuristics} shows the average number of steps needed to reach the
target by each of the strategies. The strategy built using the standard
$\arg\min$ operator does not reach the target within the first $200$ steps in
any of the $10000$ trials. The goal-leaning $\arg\min$ operator itself helps a
lot to navigate the agent towards the goal. However, it still relies on rare
events at some places. Setting $\theta=0.3$ helps to avoid these situations as
the unlikely outcomes are not considered any more, and finally $\theta=0.5$
forces the agent to use strong actions almost exclusively. While using
thresholds led to a better ERT in this particular environment, the result
might not hold in general. The best choice of threshold solely depends on the
environment model and the exact probabilities of outcomes.


\section{Conclusion \& future work}

We presented consumption Markov decision processes --- models for stochastic
environments with resource constraints --- and we showed that strategy
synthesis for qualitative objectives is efficient. In particular, our
algorithms that solve synthesis for almost-sure reachability and almost-sure
Büchi objective in CMDPs, work in time polynomial with respect to the
representation of the input CMDP. In addition, we presented two heuristics
that can significantly improve the expected time needed to reach a target in
realistic examples. These heuristics improve the utility of the presented
algorithms for planning in stochastic environments under resource constraints.
The experimental evaluation of the suggested methods confirmed that direct
analysis of CMDPs in out tool is faster than analysis of an equivalent MDP
even when performed by the state-of-the-art tool \Storm (with the exception of
very small models).

Possible directions for the future work include extensions to quantitative
analysis (e.g. minimizing the expected resource consumption or reachability
time), stochastic games, or partially observable setting.

\noindent\textbf{Acknowledgements:} We acknowledge the kind help of Tomáš Brázdil, Vojtěch Forejt, David Klaška, and Martin Kučera in the discussions leading to this paper.

\bibliographystyle{ieeetran}
\bibliography{journal}

\newpage

\begin{IEEEbiography}[{\includegraphics[width=1in,height=1.25in,clip,keepaspectratio]{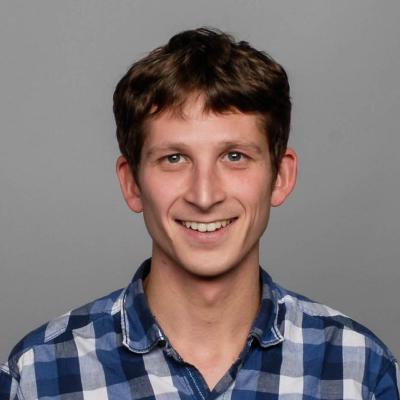}}]{František Blahoudek} is a postdoctoral researcher at the Faculty of Information Technology, Brno University of Technology, Czech Republic. He was a postdoctoral researcher in the group of Ufuk Topcu at the University of Texas at Austin. He received his Ph.D. degree from the Masaryk University, Brno in 2018. His research focuses on automata in formal methods and on planning under resource constraints.
\end{IEEEbiography}

\begin{IEEEbiography}[{\includegraphics[width=1in,height=1.25in,clip,keepaspectratio]{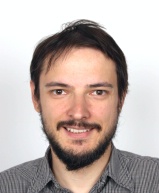}}]{Petr Novotný} is an assistant professor at the Faculty of Informatics, Masaryk University, Czech Republic. He received his Ph.D. degree from Masaryk University in 2015. His research focuses on automated analysis of probabilistic program, application of formal methods in the domains of planning and reinforcement learning, and on the theoretical foundations of probabilistic verification.
\end{IEEEbiography}

\begin{IEEEbiography}[{\includegraphics[width=1in,height=1.25in,clip,keepaspectratio]{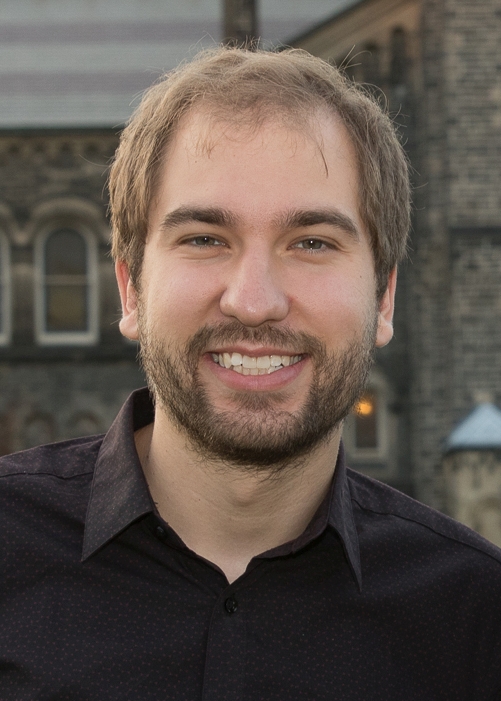}}]{Melkior Ornik} is an assistant professor in the Department of Aerospace Engineering and the Coordinated Science Laboratory at the University of Illinois at Urbana-Champaign. He received his Ph.D. degree from the University of Toronto in 2017. His research focuses on developing theory and algorithms for learning and planning of autonomous systems operating in uncertain., complex and changing environments, as well as in scenarios where only limited knowledge of the system is available.
\end{IEEEbiography}

\begin{IEEEbiography}[{\includegraphics[width=1in,height=1.25in,clip,keepaspectratio]{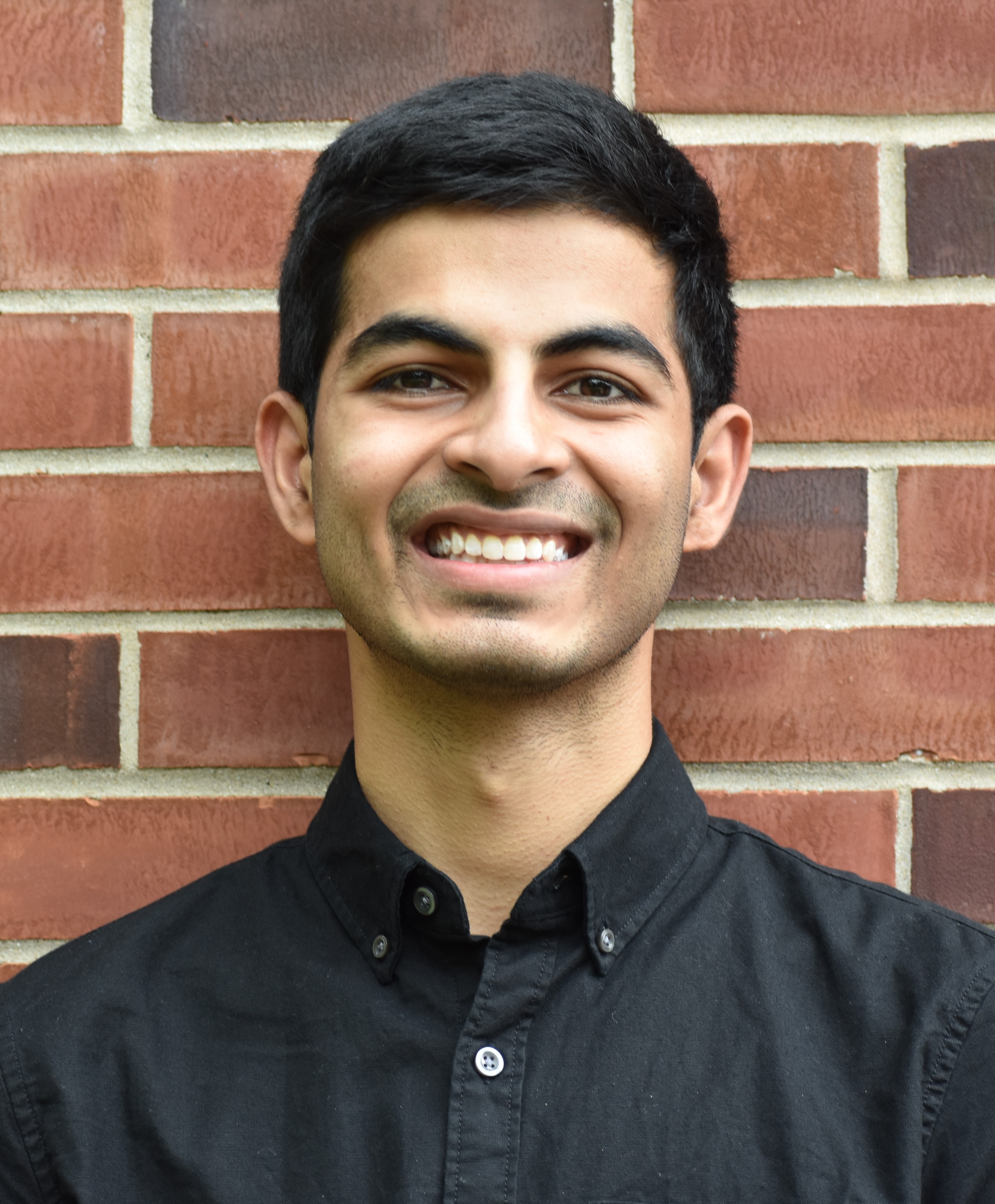}}]{Pranay Thangeda} is a graduate student in the Department of Aerospace Engineering and the Coordinated Science Laboratory at the University of Illinois at Urbana-Champaign. He received his M.S. degree in Aerospace Engineering from the University of Illinois at Urbana-Champaign in 2020. His research focuses on developing algorithms that exploit side information for efficient planning and learning in unknown environments.
\end{IEEEbiography}

\begin{IEEEbiography}[{\includegraphics[width=1in,height=1.25in,clip,keepaspectratio]{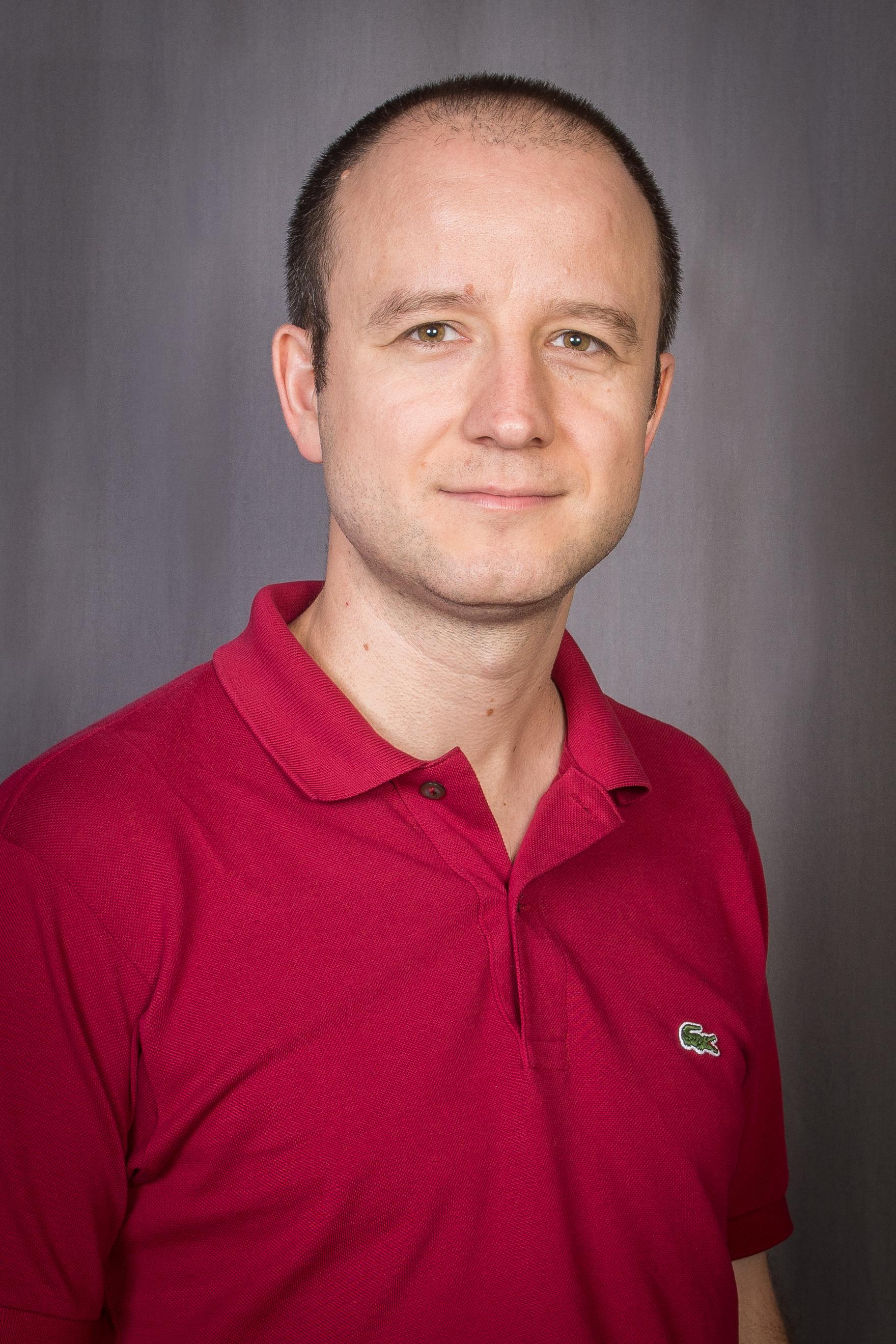}}]{Ufuk Topcu} is an associate professor in the Department of Aerospace Engineering and Engineering Mechanics and the Oden Institute at The University of Texas at Austin. He received his Ph.D. degree from the University of California at Berkeley in 2008. His research focuses on the theoretical, algorithmic, and computational aspects of design and verification of autonomous systems through novel connections between formal methods, learning theory and controls.
\end{IEEEbiography}

\vfill

\appendix

\subsection{Proof of \Cref*{lem:posreach-iterate}}
\label{app:posreach-iterate}
\begin{proof}[Proof of \cref{lem:posreach-iterate}]
We proceed by induction on $ i $. The base case is clear. Now assume that the statement
holds for some $ i \geq 0 $. Fix any $s$. Denote by $b =
\veccomp{\mathcal{B}^{i+1}(\vec{y}_\target)}{s}$ and
$d=\veccomp{\ml{{\reachability^{i+1}}}{\pos}}{s}$. We show that $b = d$. The
equality trivially holds whenever $s\in \target$, so in the remainder of the
proof we assume that $s\not \in \target$.

We first prove that $b\geq d$. If $b=\infty$, this is clearly true. Otherwise,
let $a_{\min}$ be the action minimizing
$\SV{\mathcal{B}^{i}(\vec{y}_T)}{s}{a_{\min}}$ (which equals $b$ if $ s\not
\in \reloads $) and let $t_{\min} \in \Succ(s,a_{\min})$ be the successor with
the lowest hope value. We denote by $p$ the value
$\veccomp{\ml{\reachability^{i}}{\pos}}{t_{\min}} \geq
\veccomp{\ml{\safety}{}}{t_{\min}}$ in the following two paragraphs. By
induction hypothesis, there exists a strategy $\sigma_1$ such that
$\sdsat{\sigma_1}{t_{\min}}{p}{\reachability^i}{\pos}$, and there also exists
a strategy $\sigma_2$ such that $\sdsat{\sigma_2}{t}{l}{\safety}{}$ for all
other successors $t\in\Succ(s,a_{min}), t\neq t_{\min}$ with
$l=\veccomp{\ml{\safety}{}}{t}$. We now fix a run $\loadedpath{\rho}{p}$ as a
run from $\lcompatible{\sigma_1}{t_{\min}}{p}$ that reaches $\target$ in at
most $i$ steps (which must exists).

Consider now a strategy $\pi $ which, starting in $s$, plays $a_{\min}$. If
the outcome of $a_{\min}$ is $t_{\min}$, the strategy $\pi$ starts to mimic
$\sigma_1$, otherwise it starts to mimic $\sigma_2$. By definition of
$\SVname$ we have that $b \geq \cons(s,a_{\min}) + l$ for $l =
\veccomp{\ml{\safety}{}}{t}$ for all $t\in\Succ(s,a_{\min})$ (including
$t_{\min}$) and thus $\sdsat{\pi}{s}{b}{\safety}{}$. The loaded run
$\loadedpath{s}{b}a_{\min}t_{\min}\histconc \rho \in \lcompatible{\pi}{s}{b}$
is safe (as $b \geq \cons(s, a_{\min}) + p$) and thus it is the witness that
$\sdsat{\pi}{s}{b}{\reachability^{i+1}}{\pos}$.

Now we prove that $ b \leq d $. This clearly holds if $ d = \infty $, so in
the remainder of the proof we assume $ d\leq \Ca$. By the definition of
$d$ there exists a strategy $\sigma$ such that
$\sdsat{\sigma}{s}{d}{\reachability^{i+1}}{\pos}$ Let
$a=\sigma(\loadedpath{s}{d})$ be the action selected by $\sigma$ in the first
step when starting in $s$ loaded by $d$.
For each $t\in \Succ(s,a)$ we assign a number $d_t$ defined as $d_t=0$ if
$t\in\reloads$ and $d_t=\lreslev{d}{sat}=d - \cons(s,a)$ otherwise.

We finish the proof by proving these two claims:
\begin{compactenum}[(1)]
\item It holds $\SV{\mathcal{B}^{i}(\vec{y}_\target)}{s}{a} \leq \cons(s, a) +
\max_{t\in \Succ(s,a)} d_t$.\label{item:sv-bounded}%
\item If $s \not \in \reloads$, then $ \cons(s,a) + \max_{t\in \Succ(s,a)}
d_t \leq d$.\label{item:sv-bounded-ii}%
\end{compactenum}
Let us first see why these claims are indeed sufficient. From
(\ref{item:sv-bounded}) we get $
\veccomp{\mathcal{A}(\mathcal{B}^{i}(\vec{y}_\target))}{s} \leq \cons(s, a) +
\max_{t\in \Succ(s,a)} d_t \leq \Ca$ (from the definition of
$\reslevs{d}{sat}$). If $ s\in \reloads $, then it follows that $b =
\veccomp{\trunc{\mathcal{A}(\mathcal{B}^{i}(\vec{y}_\target))}_\reloads^\Ca}{s} = 0 \leq d $. If $ s\not \in \reloads $, then $b = \veccomp{\trunc{\mathcal{A}(\mathcal{B}^{i}(\vec{y}_\target))}_\reloads^\Ca}{s} = \veccomp{\mathcal{A}(\mathcal{B}^{i}(\vec{y}_T))}{s} \leq \cons(s, a) + \max_{t\in \Succ(s,a)} d_t \leq d$, the first inequality shown above and the second coming from (\ref{item:sv-bounded-ii}).

Let us first prove (1.). We denote by $\tau$ the strategy such that
for all histories $\hist$ we have $\tau(\hist) = \sigma(sa\hist)$. For each
$t\in \Succ(s,a)$, we have $\sdsat{\tau}{t}{d_t}{\safety}{}$. Moreover, there
exists $q \in \Succ(s,a)$ such that
$\sdsat{\tau}{q}{d_q}{\reachability^i}{\pos}$ (since $s\notin\target$); hence,
by induction hypothesis it holds $\veccomp{\mathcal{B}^{i}(\vec{y}_T)}{q} \leq
d_q$. From this and from the definition of  $\SVname$ we get
\begin{align*}
\SV{\mathcal{B}^{i}(\vec{y}_T)}{s}{a}
  & \leq \cons(s,a) + \shope{\mathcal{B}^{i}(\vec{y}_T)}{s}{a}{q} \\
  & \leq \cons(s,a) + \!\!\!\!\max_{\substack{t\in\Succ(s,a)\\t\neq q}}%
    \!\!\{
      \veccomp{\mathcal{B}^{i}(\vec{y}_T)}{q}, \veccomp{\ml{\safety}{}}{s}%
    \}\\
  & \leq \cons(s,a) + \max_{t\in \Succ(s,a)} d_t.
\end{align*}

To finish, (\ref{item:sv-bounded-ii}) follows immediately from the definition
of $d_t$ and the fact that $\lreslev{d}{sat}$ is always bounded from above by
$d-\cons(s,a)$ for $s\notin \reloads$.
\end{proof}

\subsection{Proof of \Cref*{lem:posreach-bound}}
\label{app:posreach-bound}
\begin{proof}[Proof of \cref{lem:posreach-bound}]
By \Cref{lem:posreach-iterate}, it suffices to show that $
\ml{\reachability}{\pos} = \ml{\reachability^K}{\pos}$. To show this, fix any
state $s$ such that $\veccomp{\ml{\reachability}{\pos}}{s} < \infty$. For the
sake of succinctness, we denote $\veccomp{\ml{\reachability}{\pos}}{s}$ by
$d$. To each strategy $\pi$ such that $\sdsat{\pi}{s}{d}{\reachability}{\pos}$
we assign a \emph{\reachindex} that is the infimum of all $i$ such that
$\sdsat{\pi}{s}{d}{\reachability^i}{\pos}$. Let $\sigma$ be a strategy such
that $\sdsat{\sigma}{s}{d}{\reachability}{\pos}$ with the minimal \reachindex{} $k$.
We show that the $k \leq K$.

We proceed by a suitable ``strategy surgery'' on $\sigma$. Let
$\hist$ be a history produced by $\sigma$ from $s$ of length $k$
whose last state belongs to $\target$. Assume, for the sake of
contradiction, that $k > K$. This can only be if at least one of the
following conditions hold:
\begin{compactenum}[(a)]
\item Some reload state is visited twice on $\hist$, i.e. there are $0
\leq j < l \leq k + 1$ such that $ \rstate[\hist]{l} =
\rstate[\hist]{j} \in \reloads $, or%
\item some state is visited twice with no intermediate visits to a
reload state; i.e., there are $ 0 \leq j < l \leq k + 1$ such that $
\rstate[\hist]{j} = \rstate[\hist]{l} $ and $ \rstate[\hist]{h}
\not\in \reloads $ for all $ j < h <l $.
\end{compactenum}
Indeed, if none of the conditions hold, then the reload states
partition $ \hist $ into at most $ |\reloads|+1 $ segments, each
segment containing non-reload states without repetition. This would
imply $ k = \len{\hist} \leq K$.

In both cases (a) and (b) we can arrive at a contradiction using essentially
the same argument. Let us illustrate the details on case (a): Consider a
strategy $\pi$ such that for every history of the form $\hist\pref{j}
\histconc\histpr $ for a suitable $\histpr$ we have $\pi(\hist\pref{j}
\histconc\histpr) = \sigma(\hist\pref{l}\histconc \histpr)$; on all other
histories, $\pi$ mimics $\sigma$. Clearly $\sdsat{\pi}{s}{d}{\safety}{}$:
the behavior changed only for histories with $\hist\pref{j}$ as a prefix and
for each suitable $\histpr$ we have
$\lreslev{d}{\hist\pref{j}\histconc\histpr} =
\lreslev{d}{\hist\pref{l}\histconc\histpr}$ due to the fact that
$\rstate[\hist]{j}=\rstate[\hist]{l}$ is a reload state. Moreover, we have
that $\loadedpath{\hist}{d}\pref{j}\histconc\hist\infix{l}{k}$ created by
$\pi$ reaches $\target$ in $k' = k - (l-j)< k$ steps which is the
\reachindex{} of $\pi$. We reached a contradiction with the choice of
$\sigma$.

For case (b), the only difference is that now the resource level after
$\hist\pref{j}\histconc\histpr$ can be higher then the one of
$\hist\pref{l}\histconc\histpr$ due to the removal of the intermediate
non-reloading cycle. Since we need to show that the energy level never
drops below 0, the same argument works.
\end{proof}

\end{document}